%% file: main.tex
\documentclass{article}
\PassOptionsToPackage{numbers, compress}{natbib}

\usepackage{fullpage}

\usepackage[utf8]{inputenc} \usepackage[T1]{fontenc}    \usepackage{hyperref}       \usepackage{url}            \usepackage{booktabs}       \usepackage{amsfonts}       \usepackage{nicefrac}       \usepackage{microtype}      \usepackage{natbib}
\input{math_commands.tex}

\usepackage{hyperref}
\usepackage{url}
\usepackage{amsmath}
\usepackage{amssymb}
\usepackage{amsthm}
\usepackage[utf8]{inputenc} \usepackage[T1]{fontenc}    \usepackage{hyperref}       \usepackage{url}            \usepackage{booktabs}       \usepackage{amsfonts}       \usepackage{nicefrac}       \usepackage{microtype}      
\usepackage{amsmath,amssymb,amsthm}
\usepackage{color}
\usepackage{hyperref}
\usepackage{graphicx}
\usepackage{subcaption}
\usepackage{wrapfig}
\usepackage{bbm}

\input{macro}

\title{Regularization Matters: Generalization and Optimization of Neural Nets v.s. their Induced Kernel}

\author{%
	Colin Wei
    \thanks{Stanford University, email:
	\texttt{colinwei@stanford.edu}} 
	\and
	Jason D. Lee
	\thanks{
	Princeton University, email:
	\texttt{jasonlee@princeton.edu}}
	\and
	Qiang Liu
	\thanks{
	University of Texas at Austin, email:
	\texttt{lqiang@cs.texas.edu}}
	\and
	Tengyu Ma
	\thanks{
	Stanford University, email:
	\texttt{tengyuma@stanford.edu}}}
\begin{document}
	
	\maketitle
	\input{intro.tex}

	\input{kernel_vs_nn}

	\input{wasserstein_opt.tex}
	\input{margin}

\input{experiments.tex}

	\input{conclusion.tex}

	\bibliography{ref}
	\bibliographystyle{plainnat}
	\newpage
	\appendix
		\input{proof_kernel_lower_bound}

			\input{proofmargin}

				\input{rademacher}

		\input{proofopt}

			\input{additional_experiments}

\end{document}

%% file: math_commands.tex
\usepackage{amsmath,amsfonts,bm}

\newcommand{\relu}{\textup{relu}}

\newcommand{\cD}{\mathcal{D}}
\newcommand{\ind}[1]{\mathbf{1}[#1]}

\def\eqref#1{equation~\ref{#1}}

\def\1{\bm{1}}

\def\eps{{\epsilon}}

\DeclareMathAlphabet{\mathsfit}{\encodingdefault}{\sfdefault}{m}{sl}
\SetMathAlphabet{\mathsfit}{bold}{\encodingdefault}{\sfdefault}{bx}{n}

\newcommand{\pdata}{p_{\rm{data}}}

\newcommand{\E}{\mathbb{E}}

\newcommand{\R}{\mathbb{R}}

\DeclareMathOperator*{\argmax}{arg\,max}
\DeclareMathOperator*{\argmin}{arg\,min}

%% file: macro.tex
\def\shownotes{1}  \ifnum\shownotes=1

\newtheorem{theorem}{Theorem}[section]
\newtheorem{proposition}[theorem]{Proposition}
\newtheorem{lemma}[theorem]{Lemma}
\newtheorem{corollary}[theorem]{Corollary}

\newtheorem{claim}[theorem]{Claim}

\newtheorem{example}[theorem]{Example}

\newtheorem{assumption}[theorem]{Assumption}

\numberwithin{equation}{section}

\newcommand{\genparam}{\Theta}
\newcommand{\bgenparam}{\bar{\Theta}}

\newcommand{\param}{\theta}
\newcommand{\bparam}{\bar{\theta}}

\newcommand{\ones}{\vec{1}}

\newcommand{\subj}{\text{subject to }}
\newcommand{\bdir}{\bar{u}}

\newcommand{\dir}{u}
\newcommand{\weight}{w}
\newcommand{\sph}{\mathbb{S}^{d - 1}}
\newcommand{\supp}{\text{supp}}
\newcommand{\hyp}{\mathcal{F}}

\newcommand{\erad}{\hat{\mathfrak{R}}}

\newcommand{\sphd}{\mathbb{S}^d}

\newcommand{\dist}{\rho}
\newcommand{\spaceone}{\mathcal{L}^1}
\newcommand{\spacetwo}{\mathcal{L}^2}
\newcommand{\spaceinf}{\mathcal{L}^{\infty}}

\newcommand{\poly}{\textup{poly}}
\newcommand{\one}{\mathbbm{1}}

%% file: intro.tex
\begin{abstract}

Recent works have shown that on sufficiently over-parametrized neural nets, gradient descent with relatively large initialization optimizes a prediction function in the RKHS of the Neural Tangent Kernel (NTK).
This analysis leads to global convergence results but does not work when there is a standard $\ell_2$ regularizer, which is useful to have in practice. We show that sample efficiency can indeed depend on the presence of the regularizer: we construct a simple distribution in $d$ dimensions which the optimal regularized neural net learns with $O(d)$ samples but the NTK requires $\Omega(d^2)$ samples to learn. To prove this, we establish two analysis tools: i) for multi-layer feedforward ReLU nets, we show that the global minimizer of a weakly-regularized cross-entropy loss is the max normalized margin solution among all neural nets, which generalizes well;
ii) we develop a new technique for proving lower bounds for kernel methods, which relies on showing that the kernel cannot focus on informative features. Motivated by our generalization results, we study whether the regularized global optimum is attainable. We prove that for infinite-width two-layer nets, noisy gradient descent optimizes the regularized neural net loss to a global minimum in polynomial iterations.\end{abstract}

\section{Introduction}

In deep learning, over-parametrization refers to the widely-adopted technique of using more parameters than necessary~\citep{krizhevsky2012imagenet,livni2014computational}. Over-parametrization is crucial for successful optimization, and a large body of work has been devoted towards understanding why. One line of recent works~\citep{daniely2017sgd,li2018learning,du2018gradient,du2018gradientb,allen2018convergence,zou2018stochastic,jacot2018neural,arora2019fine,chizat2018note,yang2019scaling}~offers an explanation that invites analogy with kernel methods, proving that with sufficient over-parameterization and a certain initialization scale and learning rate schedule, gradient descent essentially learns a linear classifier on top of the initial random features. For this same setting,~\citet{daniely2017sgd,du2018gradient,du2018gradientb,jacot2018neural,arora2019fine,arora2019exact} make this connection explicit by establishing that the prediction function found by gradient descent is in the span of the training data in a reproducing kernel Hilbert space (RKHS) induced by the Neural Tangent Kernel (NTK). The generalization error of the resulting network can be analyzed via the Rademacher complexity of the kernel method.

These works provide some of the first algorithmic results for the success of gradient descent in optimizing neural nets; however, the resulting generalization error is only as good as that of fixed kernels~\citep{arora2019fine}. On the other hand, the equivalence of gradient descent and NTK is broken if the loss has an explicit regularizer such as weight decay. 

In this paper, we study the effect of an explicit regularizer on neural net generalization via the lens of margin theory. We first construct a simple distribution on which the two-layer network optimizing explicitly regularized logistic loss will achieve a large margin, and therefore, good generalization. On the other hand, any prediction function in the span of the training data in the RKHS induced by the NTK will overfit to noise and therefore achieve poor margin and bad generalization.

\begin{theorem}[Informal version of Theorem~\ref{thm:comparison}]
	\label{thm:informal_comparison}
	Consider the setting of learning  the distribution $\cD$ defined in Figure~\ref{fig:cDvis} using a two-layer network with relu activations with the goal of achieving small generalization error. Using $o(d^2)$ samples, no function in the span of the training data in the RKHS induced by the NTK can succeed. On the other hand, the global optimizer of the $\ell_2$-regularized logistic loss can learn $\cD$ with $O(d)$ samples.
\end{theorem}

The full result is stated in Section~\ref{sec:kernel_vs_nn}. The intuition is that regularization allows the neural net to obtain a better margin than the fixed NTK kernel and thus achieve better generalization. Our sample complexity lower bound for NTK applies to a broad class of losses including standard 0-1 classification loss and squared $\ell_2$. To the best of our knowledge, the proof techniques for obtaining this bound are novel and of independent interest (see our proof overview in Section~\ref{sec:kernel_vs_nn}). In Section~\ref{sec:experiments}, we confirm empirically that an explicit regularizer can indeed improve the margin and generalization.

\citet{yehudai2019power}~also prove a lower bound on the learnability of neural net kernels. They show that an approximation result that $\Omega(\exp(d))$ random relu features are required to fit a single neuron in $\ell_2$ squared loss, which lower bounds the amount of over-parametrization necessary to approximate a single neuron. In contrast, we prove sample-complexity lower bounds which hold for both classification and $\ell_2$ loss even with \textit{infinite over-parametrization}. 

Motivated by the provably better generalization of regularized neural nets for our constructed instance, in Section~\ref{sec:optimization}~we study their optimization, as the previously cited results only apply when the neural net behaves like a kernel. We show optimization is possible for infinite-width regularized nets.

\begin{figure}
	\begin{minipage}[c]{0.16\textwidth}
	~~~~~~~~~~~~~~~~~\includegraphics[width=\textwidth]{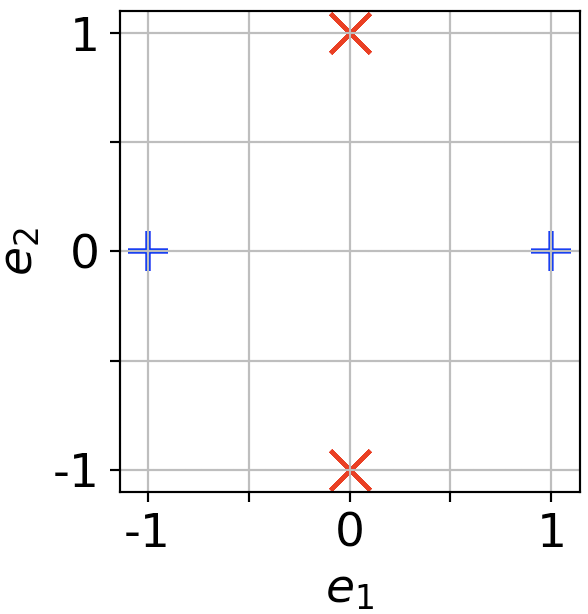}
	\end{minipage}\hfill
	\begin{minipage}[c]{0.6\textwidth}
		\caption{Datapoints from $\cD$ have first two coordinates displayed above, with red and blue denoting labels of -1, +1, respectively. The remaining coordinates are uniform in $\{-1, +1\}^{d - 2}$.}
		\label{fig:cDvis}	
	\end{minipage}
	\vspace{-0.5cm}
\end{figure}

\begin{theorem}[Informal, see Theorem~\ref{thm:polytimeopt-noparam}]
	For infinite-width two layer networks with $\ell_2$-regularized loss, noisy gradient descent finds a global optimizer in a polynomial number of iterations.
\end{theorem}

This improves upon prior works \citep{mei2018mean, chizat2018global,sirignano2018mean,rotskoff2018neural} which study optimization in the same infinite-width limit but do not provide polynomial convergence rates. (See more discussions in Section~\ref{sec:optimization}.)

To establish Theorem~\ref{thm:informal_comparison}, we rely on tools from margin theory. In Section~\ref{sec:margin}, we prove a number of results of independent interest regarding the margin of a regularized neural net. We show that the global minimum of weakly-regularized logistic loss of any homogeneous network (regardless of depth or width) achieves the max normalized margin among all networks with the same architecture (Theorem~\ref{thm:binary_margin}). By ``weak'' regularizer, we mean that the coefficient of the regularizer in the loss is very small (approaching 0). By combining with a result of \cite{golowich2017size}, we conclude that the minimizer enjoys a width-free generalization bound depending on only the inverse normalized margin (normalized by the norm of the weights) and depth  (Corollary~\ref{cor:margingeneralization}). This explains why optimizing the $\ell_2$-regularized loss typically used in practice can lead to parameters with a large margin and good generalization. We further note that the maximum possible margin is non-decreasing in the width of the architecture, so the generalization bound of Corollary~\ref{cor:margingeneralization} improves as the size of the network grows (see Theorem~\ref{thm:marginnondecreasing}). Thus, even if the dataset is already separable, it could still be useful to further over-parameterize to achieve better generalization.

Finally, we empirically validate several claims made in this paper in Section~\ref{sec:experiments}. First, we confirm on synthetic data that neural networks do generalize better with an explicit regularizer vs. without. Second, we show that for two-layer networks, the test error decreases and margin increases as the hidden layer grows, as predicted by our theory.

\subsection{Additional Related Work}
\citet{zhang2016understanding} and \citet{neyshabur2017exploring} show that neural network generalization defies conventional explanations and requires new ones. \citet{neyshabur2014search} initiate the search for the ``inductive bias'' of neural networks towards solutions with good generalization. Recent papers \citep{hardt2015train,brutzkus2017sgd,chaudhari2016entropy} study inductive bias through training time and sharpness of local minima. \citet{neyshabur2015path} propose a steepest descent algorithm in a geometry invariant to weight rescaling and show this improves generalization. \citet{morcos2018importance} relate generalization to the number of ``directions'' in the neurons. Other papers \citep{gunasekar2017implicit,soudry2018the,nacson2018convergence,gunasekar2018implicit, li2018algorithmic,gunasekar2018characterizing,li2018algorithmic,ji2018risk} study implicit regularization towards a specific solution. \citet{ma2017implicit} show that implicit regularization helps gradient descent avoid overshooting optima. \citet{rosset2004boosting,rosset2004margin} study linear logistic regression with weak regularization and show convergence to the max margin. In Section~\ref{sec:margin}, we adopt their techniques and extend their results.

A line of work initiated by \citet{neyshabur2015norm} has focused on deriving tighter norm-based Rademacher complexity bounds for deep neural networks \citep{bartlett2017spectrally,neyshabur2017pac,golowich2017size} and new compression based generalization properties \citep{arora2018stronger}.~\citet{bartlett2017spectrally} highlight the important role of normalized margin in neural net generalization.~\citet{wei2019data}~prove generalization bounds depending on additional data-dependent properties.~\citet{dziugaite2017computing} compute non-vacuous generalization bounds from PAC-Bayes bounds. \citet{neyshabur2018towards} investigate the Rademacher complexity of two-layer networks and propose a bound that is decreasing with the distance to initialization. \citet{liang2018just} and \citet{belkin2018understand} study the generalization of kernel methods.

For optimization, \citet{soudry2016no} explain why over-parametrization can remove bad local minima. \citet{safran2016quality} show over-parametrization can improve the quality of a random initialization. \citet{haeffele2015global}, \citet{nguyen2017loss}, and \citet{venturi2018neural} show that for sufficiently overparametrized networks, all local minima are global, but do not show how to find these minima via gradient descent. \citet{du2018power} show for two-layer networks with quadratic activations, all second-order stationary points are global minimizers. \citet{arora2018optimization} interpret over-parametrization as a means of acceleration. \citet{mei2018mean,chizat2018global,sirignano2018mean,dou2019training,mei2019mean} analyze a distributional view of over-parametrized networks. \citet{chizat2018global} show that Wasserstein gradient flow converges to global optimizers under structural assumptions. We extend this to a polynomial-time result.

Finally, many papers have shown convergence of gradient descent on neural nets~\citep{allen2018convergence,allen2018learning,li2018learning,du2018gradient,du2018gradientb,arora2019fine,zou2018stochastic,cao2019generalization,jacot2018neural,chizat2018note} using analyses which prove the weights do not move far from initialization. These analyses do not apply to the regularized loss, and our experiments in Section~\ref{sec:additionalexp}~suggest that moving away from the initialization is important for better test performance.

Another line of work takes a Bayesian perspective on neural nets. Under an appropriate choice of prior, they show an equivalence between the random neural net and Gaussian processes in the limit of infinite width or channels~\citep{neal1996priors,williams1997computing,lee2017deep,matthews2018gaussian,garriga2018deep,novak2018bayesian}. This provides another kernel perspective of neural nets.

\citet{yehudai2019power,chizat2018note}~also argue that the kernel perspective of neural nets is not sufficient for understanding the success of deep learning.~\citet{chizat2018note}~argue that the kernel perspective of gradient descent is caused by a large initialization and does not necessarily explain the empirical successes of over-parametrization.~\citet{yehudai2019power}~prove that $\Omega(\exp(d))$ random relu features cannot approximate a single neuron in squared error loss. In comparison, our lower bounds are for the sample complexity rather than width of the NTK prediction function and apply even with infinite over-parametrization for both classification and squared loss. 

\subsection{Notation}
\label{subsec:notation}
Let $\R$ denote the set of real numbers. 
We will use $\| \cdot \|$ to indicate a general norm, with $\|\cdot\|_2$ denoting the $\ell_2$ norm and $\|\cdot\|_F$ the Frobenius norm. We use $\ \bar{}\ $ on top of a symbol to denote a unit vector: when applicable, $\bar{u} \triangleq u/\|u\|$, with the norm $\|\cdot\|$ clear from context. Let $\mathcal{N}(0, \sigma^2)$ denote the normal distribution with mean 0 and variance $\sigma^2$. For vectors $u_1 \in\R^{d_1}$, $u_2 \in\R^{d_2}$, we use the notation $(u_1, u_2) \in \R^{d_1 + d_2}$ to denote their concatenation. We also say a function $f$ is $a$-homogeneous in input $x$ if $f(cx) = c^a f(x)$ for any $c$, and we say $f$ is $a$-positive-homogeneous if there is the additional constraint $c > 0$. We reserve the symbol $X = [x_1,\dots, x_n]$ to denote the collection of datapoints (as a matrix), and $Y = [y_1,\dots, y_n]$ to denote labels. We use $d$ to denote the dimension of our data. We will use the notations $a \lesssim b$, $a \gtrsim b$ to denote less than or greater than up to a universal constant, respectively, and when used in a condition, to denote the existence of such a constant such that the condition is true. Unless stated otherwise, $O(\cdot), \Omega(\cdot)$ denote some universal constant in upper and lower bounds. The notation $\poly$ denotes a universal constant-degree polynomial in the arguments.

%% file: kernel_vs_nn.tex
\section{Generalization of Regularized Neural Net vs. NTK Kernel}
\label{sec:kernel_vs_nn}

We will compare neural net solutions found via regularization and methods involving the NTK and construct a data distribution $\mathcal{D}$ in $d$ dimensions which the neural net optimizer of regularized logistic loss learns with sample complexity $O(d)$. The kernel method will require $\Omega(d^2)$ samples to learn. 

We start by describing the distribution $\cD$ of examples $(x,y)$. Here $e_i$ is the i-th standard basis vector and we use $x^\top e_i$ to represent the $i$-coordinate of $x$ (since the subscript is reserved to index training examples). First, for any $ k \ge 3, x^\top e_k \sim \{-1,+1\}$ is a uniform random bit, and for $x^\top e_1, x^\top e_2$ and $y$, choose
\begin{align}
\begin{array}{llll}
y = +1,  & x^\top e_1 = +1, & x^\top e_2 = 0& \textup{ w/ prob. } 1/4 \\
y = +1, & x^\top e_1 = -1, &x^\top e_2 = 0 & \textup{ w/ prob. } 1/4 \\
y = -1, &x^\top e_1 = 0, &x^\top e_2 = +1 & \textup{ w/ prob. } 1/4 \\
y = -1, &x^\top e_1 = 0, &x^\top e_2 = -1 & \textup{ w/ prob. } 1/4 \\
\end{array}\label{eq:comparison_distribution}
\end{align}
The distribution $\cD$ contains all of its signal in the first 2 coordinates, and the remaining $d - 2$ coordinates are noise. We visualize its first 2 coordinates in Figure~\ref{fig:cDvis}.

Next, we formally define the two layer neural net with relu activations and its associated NTK. We parameterize a two-layer network with $m$ units by last layer weights $w_1,\ldots,w_m \in\R$ and weight vectors $u_1, \ldots, u_m \in\R^d$.  We denote by $\Theta$ the collection of parameters and by $\theta_j$ the unit-$j$ parameters $(u_j, w_j)$. The network computes 
$
f^{\textup{NN}}(x;\Theta) \triangleq \sum_{j =1}^m w_j[u_j^\top x]_+
$,
where $[\cdot]_+$ denotes the~\relu~activation. For binary labels $y_1,\ldots,y_n \in\{-1,+1\}$, the $\ell_2$ regularized logistic loss is
\begin{align}
L_{\lambda} (\Theta) \triangleq \frac{1}{n}\sum_{i= 1}^n \log(1 + \exp(-y_i f^{\textup{NN}}(x_i; \Theta))) + \lambda \|\Theta\|_F^2
\label{eq:reg-train-loss}
\end{align}
Let $\Theta_{\lambda} \in \argmin_{\Theta} L_{\lambda}(\Theta)$ be its global optimizer. Define the NTK kernel associated with the architecture (with random weights): 
\begin{align*}
K(x',x)= \E_{w\sim \mathcal{N}(0, r_w^2), u \sim\mathcal{N}(0, r_u^2I)} \left[\left\langle \nabla_{\theta} f^{\textup{NN}}(x; \Theta), \nabla_{\theta} f^{\textup{NN}}(x'; \Theta)\right\rangle\right]
\end{align*} 
where $\nabla_{\theta}f^{\textup{NN}}(x; \Theta)= (w\one(x^\top u \ge 0) x, [x^\top u]_+)$ is the gradient of the network output with respect to a generic hidden unit, and $r_w, r_u$ are relative scaling parameters. Note that the typical NTK is realized specifically with scales $r_w = r_u = 1$, but our bound applies for all choices of $r_w, r_u$. 

For coefficients $\beta$, we can then define the prediction function $f^{\textup{kernel}}(x;\beta)$ in the RKHS induced by $K$ as 
$
f^{\textup{kernel}}(x;\beta) \triangleq \sum_{i = 1}^n \beta_i K(x_i, x) 
$.
For example, such a classifier would be attained by running gradient descent on squared loss for a wide network using the appropriate random initialization (see~\citep{jacot2018neural,du2018gradient,du2018gradientb,arora2019fine}). We now present our comparison theorem below and fill in its proof in Section~\ref{sec:proof_comparison}.
\begin{theorem}\label{thm:comparison}
			Let $\cD$ be the distribution defined in equation~\ref{eq:comparison_distribution}. 
With probability $1 - d^{-5}$ over the random draw of $n \lesssim d^2$ samples $(x_1, y_1), \ldots, (x_n, y_n)$ from $\mathcal{D}$, for all choices of $\beta$, the kernel prediction function $f^{\textup{kernel}}(\cdot;\beta)$ will have at least $\Omega(1)$ error:
	\begin{align*}
	\Pr_{(x, y) \sim \mathcal{D}}[f^{\textup{kernel}}(x;\beta) y \le 0]= \Omega(1)
	\end{align*}
	Meanwhile, for $\lambda \le \poly(n)^{-1}$, the regularized neural net solution $f^{\textup{NN}}(\cdot; \Theta_{\lambda})$ with at least 4 hidden units can have good generalization with $O(d^2)$ samples because we have the following generalization error bound: 
	$$
	\Pr_{(x, y) \sim \mathcal{D}}[f^{\textup{NN}}(x;\Theta_{\lambda}) y \le 0] \lesssim \sqrt{\frac{d}{n}} 
	$$ This implies a $\Omega(d)$ sample-complexity gap between the regularized neural net and kernel prediction function.
\end{theorem}
While the above theorem is stated for classification, the same $\mathcal{D}$ can be used to straightforwardly prove a $\Omega(d)$ sample complexity gap for the truncated squared loss $\ell(\hat{y};y) = \min((y - \hat{y})^2,1)$.\footnote{The truncation is required to prove generalization of the regularized neural net using standard tools.} We provide more details in Section~\ref{subsec:regression_compare}.

Our intuition of this gap is that the regularization allows the neural net to find informative features (weight vectors), that are adaptive to the data distribution and easier for the last layers' weights to separate. For example, the neurons $[e_1x]_+$, $[-e_1x]_+$, $[e_2x]_+$, $[-e_2x]_+$ are enough to fit our particular distribution. In comparison, the NTK method is unable to change the feature space and is only searching for the coefficients in the kernel space.

\textit{Proof techniques for the upper bound: } For the upper bound, neural nets with small Euclidean norm will be able to separate $\cD$ with large margin (a two-layer net with width 4 can already achieve a large margin). As we show in Section~\ref{sec:margin}, a solution with a max neural-net margin is attained by the global optimizer of the regularized logistic loss --- in fact, we show this holds for generally homogeneous networks of any depth and width (Theorem~\ref{thm:binary_margin}). Then, by the classical connection between margin and generalization~\citep{koltchinskii2002empirical}, this optimizer will generalize well.

\textit{Proof techniques for the lower bound: } On the other hand, the NTK will have a worse margin when fitting samples from $\cD$ than the regularized neural networks because NTK operates in a fixed kernel space.\footnote{There could be some variations of the NTK space depending on the scales of the initialization of the two layers, but our Theorem~\ref{thm:comparison} shows that these variations also suffer from a worse sample complexity.} However, proving that the NTK has a small margin does not suffice because the generalization error bounds which depend on margin may not be tight. 

We develop a new technique to prove lower bounds for kernel methods, which we believe is of independent interest, as there are few prior works that prove lower bounds for kernel methods. (One that does is~\citep{ng2004feature}, but their results require constructing an artificial kernel and data distribution, whereas our lower bounds are for a fixed kernel.) The main intuition is that because NTK uses infinitely many random features, it is difficult for the NTK to focus on a small number of informative features -- doing so would require a very high RKHS norm. In fact,  we show that with a limited number of examples, any function that in the span of the training examples must heavily use random features rather than informative features. The random features can collectively fit the training data, but will give worse generalization.

%% file: wasserstein_opt.tex
\section{Perturbed Wasserstein Gradient Flow Finds Global Optimizers in Polynomial Time}
\label{sec:optimization}
In the prior section, we argued that a neural net with $\ell_2$ regularization can achieve much better generalization than the NTK. Our result required attaining the global minimum of the regularized loss; however, existing optimization theory only allows for such convergence to a global minimizer with a large initialization and no regularizer. Unfortunately, these are the regimes where the neural net learns a kernel prediction function~\citep{jacot2018neural,du2018gradient,arora2019fine}.

In this section, we show that at least for infinite-width two-layer nets, optimization is not an issue: noisy gradient descent finds global optimizers of the $\ell_2$ regularized loss in polynomial iterations.

Prior work \citep{mei2018mean,chizat2018global} has shown that as the hidden layer size grows to infinity, gradient descent for a finite neural network approaches the Wasserstein gradient flow over distributions of hidden units (defined in \eqref{eq:wassflow}). With the assumption that the gradient flow converges, which is non-trivial since the space of distributions is infinite-dimensional,~\citet{chizat2018global} prove that Wasserstein gradient flow converges to a global optimizer but do not specify a rate. \citet{mei2018mean} add an entropy regularizer to form an objective that is the infinite-neuron limit of stochastic Langevin dynamics. They show global convergence but also do not provide explicit rates. In the worst case, their convergence can be exponential in dimension. In contrast, we provide explicit \textit{polynomial} convergence rates for a slightly different algorithm, perturbed Wasserstein gradient flow. 

Infinite-width neural nets are modeled mathematically as a distribution over weights:
formally, we optimize the following functional over distributions $\dist$ on $\R^{d+1}$:
$
L[\dist] \triangleq R(\int \Phi d\dist ) + \int V d\dist 
$,
where $\Phi : \R^{d + 1} \rightarrow \R^k$, $R : \R^k \rightarrow \R$, and $V : \R^{d + 1} \rightarrow \R$. $R$ and $V$ can be thought of as the loss and regularizer, respectively. In this work, we consider 2-homogeneous $\Phi$ and $V$. We will additionally require that $R$ is convex and nonnegative and $V$ is positive on the unit sphere. Finally, we need standard regularity assumptions on $R, \Phi$, and $V$:
\begin{assumption}[Regularity conditions on $\Phi$, $R$, $V$]
	\label{assu:reg}
	$\Phi$ and $V$ are differentiable as well as upper bounded and Lipschitz on the unit sphere. $R$ is Lipschitz and its Hessian has bounded operator norm.
\end{assumption}
We provide more details on the specific parameters (for boundedness, Lipschitzness, etc.) in Section~\ref{subsec:optsetup}. We note that relu networks satisfy every condition but differentiability of $\Phi$.\footnote{The~\relu~activation is non-differentiable at 0 and hence the gradient flow is not well-defined. \citet{chizat2018global} acknowledge this same difficulty with~\relu.} We can fit a $\ell_2$ regularized neural network under our framework:
\begin{example}[Logistic loss for neural networks]
	\label{ex:logisticloss}
	We interpret $\dist$ as a distribution over the parameters of the network. Let $k \triangleq n$ and $\Phi_i(\param) \triangleq \weight \phi(\dir^{\top}x_i)$ for $\param = (\weight, \dir)$. In this case, $\int \Phi d\dist$ is a distributional neural network that computes an output for each of the $n$ training examples (like a standard neural network, it also computes a weighted sum over hidden units). We can compute the distributional version of the regularized logistic loss in \eqref{eq:reg-train-loss} by setting $V(\param)  \triangleq \lambda \|\param\|_2^2$ and $R(a_1, \ldots, a_n) \triangleq \sum_{i = 1}^n \log(1 + \exp(-y_i a_i))$. 
\end{example}

We will define $L'[\dist] : \R^{d + 1} \rightarrow \R$ with $L'[\dist](\param) \triangleq \langle R'(\int \Phi d\dist), \Phi(\param)\rangle + V(\param)$ and $v[{\dist}](\param) \triangleq -\nabla_\param L'[\dist](\param)$. Informally, $L'[\dist]$ is the gradient of $L$ with respect to $\dist$, and $v$ is the induced velocity field. For the standard Wasserstein gradient flow dynamics, $\dist_t$ evolves according to 
\begin{align}
\label{eq:wassflow}
\frac{d}{dt} \dist_t = - \nabla \cdot(v[\dist_t]\dist_t)
\end{align}
where $\nabla \cdot $ denotes the divergence of a vector field. For neural networks, these dynamics formally define continuous-time gradient descent when the hidden layer has infinite size (see Theorem 2.6 of~\citep{chizat2018global}, for instance). More generally,~\eqref{eq:wassflow} is due to the formula for Wasserstein gradient flow dynamics (see for example~\citep{santambrogio2017euclidean}), which are derived via continuous-time steepest descent with respect to Wasserstein distance over the space of probability distributions on the neurons. We propose the following modified dynamics:
\begin{align}
\label{eq:noisywassersteindynamics}
\frac{d}{dt} \dist_t = -\sigma \dist_t + \sigma U^d - \nabla \cdot(v[\dist_t]\dist_t)
\end{align}
where $U^d$ is the uniform distribution on $\sphd$. In our perturbed dynamics, we add very small uniform noise over $U^d$, which ensures that at all time-steps, there is sufficient mass in a descent direction for the algorithm to decrease the objective. For infinite-size neural networks, one can informally interpret this as re-initializing a very small fraction of the neurons at every step of gradient descent. We prove convergence to a global optimizer in time polynomial in $1/\epsilon, d$, and the regularity parameters.

\begin{theorem} [Theorem~\ref{thm:polytimeopt} with regularity parameters omitted]
	\label{thm:polytimeopt-noparam}
	Suppose that $\Phi$ and $V$ are 2-homogeneous and the regularity conditions of Assumption~\ref{assu:reg} are satisfied. Also assume that from starting distribution $\dist_0$, a solution to the dynamics in \eqref{eq:noisywassersteindynamics} exists. Define $L^{\star} \triangleq \inf_{\dist} L[\dist]$. Let $\epsilon > 0$ be a desired error threshold and choose
$
	\sigma \triangleq \exp(-d \log(1/\epsilon)\poly(k, L[\dist_0] - L^{\star}))$ and 
$	 t_{\epsilon} \triangleq \frac{d^2}{\epsilon^4}\poly(\log(1/\epsilon), k, L[\dist_0] - L^{\star})$, where the regularity parameters for $\Phi$, $V$, and $R$ are hidden in the $\poly(\cdot)$. Then, perturbed Wasserstein gradient flow converges to an $\eps$-approximate global minimum in $t_\eps$ time: $$\min_{0 \le t \le t_{\epsilon}} L[\dist_t] - L^{\star} \le \epsilon$$
\end{theorem}
We state and prove a version of Theorem~\ref{thm:polytimeopt-noparam} that includes regularity parameters in Sections~\ref{subsec:optsetup}~and~\ref{subsec:optproof}. The key idea for the proof is as follows: as $R$ is convex, the optimization problem will be convex over the space of distributions $\dist$. This convexity allows us to argue that if $\dist$ is suboptimal, there either exists a descent direction $\bparam \in \sphd$ where $L'[\dist](\bparam) \ll 0$, or the gradient flow dynamics will result in a large decrease in the objective. If such a direction $\bparam$ exists, the uniform noise $\sigma U^d$ along with the 2-homogeneity of $\Phi$ and $V$ will allow the optimization dynamics to increase the mass in this direction exponentially fast, which causes a polynomial decrease in the loss. 

As a technical detail, Theorem \ref{thm:polytimeopt-noparam} requires that a solution to the dynamics exists. We can remove this assumption by analyzing a discrete-time version of \eqref{eq:noisywassersteindynamics}:
$
  \dist_{t + 1} \triangleq \dist_t + \eta(-\sigma \dist_t + \sigma U^d - \nabla \cdot(v[\dist_t] \dist_t) )
$, and additionally assuming $\Phi$ and $V$ have Lipschitz gradients. In this setting, a polynomial time convergence result also holds. We state the result in Section~\ref{subsec:discrete-time}.

An implication of our Theorem~\ref{thm:polytimeopt-noparam} is that for infinite networks, we can optimize the weakly-regularized logistic loss in time polynomial in the problem parameters and $\lambda^{-1}$. In Theorem~\ref{thm:comparison}~we only require $\lambda^{-1} = \poly(n)$; thus, an infinite width neural net can learn the distribution $\cD$ up to error $\tilde{O}(\sqrt{d/n})$ in polynomial time using noisy gradient descent. 

%% file: margin.tex
\section{Weak Regularizer Guarantees Max Margin Solutions}\label{sec:margin}
In this section, we collect a number of results regarding the margin of a regularized neural net. These results provide the tools for proving generalization of the weakly-regularized NN solution in Theorem~\ref{thm:comparison}. The key technique is showing that with small regularizer $\lambda \rightarrow 0$, the global optimizer of regularized logistic loss will obtain a maximum margin. It is well-understood that a large neural net margin implies good generalization performance~\citep{bartlett2017spectrally}.

In fact, our result applies to a function class much broader than two-layer relu nets: in Theorem~\ref{thm:binary_margin}~we show that when we add a weak regularizer to cross-entropy loss with \textit{any} positive-homogeneous prediction function, the normalized margin of the optimum converges to the max margin. For example, Theorem~\ref{thm:binary_margin} applies to feedforward relu networks of arbitrary depth and width. In Theorem~\ref{thm:multi_opterrrequired}, we bound the approximation error in the maximum margin when we only obtain an approximate optimizer of the regularized loss.  In Corollary~\ref{cor:margingeneralization}, we leverage these results and pre-existing Rademacher complexity bounds to conclude that the optimizer of the weakly-regularized logistic loss will have width-free generalization bound scaling with the inverse of the max margin and network depth. Finally, we note that the maximum possible margin can only increase with the width of the network, which suggests that increasing width can improve generalization of the solution (see Theorem~\ref{thm:marginnondecreasing}). 

We work with a family $\mathcal{F}$ of prediction functions $f(\cdot; \genparam) : \R^d \rightarrow \R$ that are $a$-positive-homogeneous in their parameters for some $a > 0$: $f(x; c\genparam) = c^a f(x; \genparam),  \forall c > 0$. We additionally require that $f$ is continuous when viewed as a function in $\genparam$. For some general norm $\|\cdot \|$ and $\lambda > 0$, we study the $\lambda$-regularized logistic loss $L_{\lambda}$, defined as
\begin{align}
\label{eq:binary_logistic_loss}
L_{\lambda}(\genparam) \triangleq \frac{1}{n}\sum_{i = 1}^n \log(1 + \exp(-y_i f(x_i;\genparam))) + \lambda \|\genparam\|^r
\end{align}
for fixed $r > 0$. Let $\genparam_{\lambda} \in \argmin L_{\lambda}(\genparam)$.\footnote{We formally show that $L_{\lambda}$ has a minimizer in Claim \ref{claim:Lhasminimizer} of Section \ref{sec:proof:margin}.} Define the normalized margin $\gamma_{\lambda}$ and max-margin $\gamma^\star$ by $\gamma_{\lambda} \triangleq \min_{i} y_i f(x_i;\bgenparam_{\lambda})$ and $\gamma^\star \triangleq \max_{\|\genparam\| \le 1} \min_{i} y_i f(x_i;\genparam)$. Let $\Theta^\star$ achieve this maximum. 

We show that with sufficiently small regularization level $\lambda$, the normalized margin $\gamma_\lambda$ approaches the maximum margin $\gamma^\star$. Our theorem and proof are inspired by the result of \citet{rosset2004boosting,rosset2004margin}, who analyze the special case when $f$ is a linear function. In contrast, our result can be applied to non-linear $f$ as long as $f$ is homogeneous. 

\begin{theorem}
	\label{thm:binary_margin}
	Assume the training data is separable by a network $f(\cdot; \Theta^\star)\in \mathcal{F}$ with an optimal normalized margin $\gamma^\star > 0$. Then, the normalized margin of the global optimum of the weakly-regularized objective (\eqref{eq:binary_logistic_loss}) converges to $\gamma^\star$ as the regularization goes to zero. Mathematically,
	$$\gamma_{\lambda} \rightarrow \gamma^\star \textup{ as } \lambda \rightarrow 0$$
\end{theorem}

An intuitive explanation for our result is as follows: because of the homogeneity, the loss $L(\genparam_{\lambda})$ roughly satisfies the following (for small $\lambda$, and ignoring parameters such as $n$):
\begin{align*}
L_{\lambda}(\genparam_{\lambda}) \approx \exp(-\|\genparam_{\lambda}\|^a \gamma_{\lambda}) + \lambda \|\genparam_{\lambda}\|^{r}\
\end{align*}
Thus, the loss selects parameters with larger margin, while the regularization favors smaller norms. The full proof of the theorem is deferred to Section~\ref{sec:proof:margin}.

Though the result in this section is stated for binary classification, it extends to the multi-class setting with cross-entropy loss. We provide formal definitions and results in Section~\ref{sec:proof:margin}. In  Theorem~\ref{thm:multi_opterrrequired}, we also show that an approximate minimizer of $L_{\lambda}$ can obtain margin that approximates $\gamma^\star$. 

Although we consider an \textit{explicit} regularizer, our result is related to recent works on algorithmic regularization of gradient descent for the \textit{unregularized} objective. Recent works show that gradient descent finds the minimum norm or max-margin solution for problems including logistic regression, linearized neural networks, and matrix factorization~\citep{soudry2018the, gunasekar2018implicit, li2018algorithmic, gunasekar2018characterizing, ji2018risk}. Many of these proofs require a delicate analysis of the algorithm's dynamics, and some are not fully rigorous due to assumptions on the iterates. To the best of our knowledge, it is an open question to prove analogous results for even two-layer relu networks. In contrast, by adding the explicit $\ell_2$ regularizer to our objective, we can prove broader results that apply to multi-layer relu networks. In the following section we leverage our result and existing generalization bounds~\citep{golowich2017size} to help justify how over-parameterization can improve generalization.
\subsection{Generalization of the Max-Margin Neural Net}\label{subsec:neural_net_gen}
We consider depth-$q$ networks with 1-Lipschitz, 1-positive-homogeneous activation $\phi$ for $q \ge 2$. Note that the network function is $q$-positive-homogeneous. Suppose that the collection of parameters $\genparam$ is given by matrices $W_1, \ldots, W_q$. For simplicity we work in the binary class setting, so the $q$-layer network computes a real-valued score
\begin{align}
\label{eq:nneq}
f^{\textup{NN}}(x;\Theta) \triangleq W_q\phi(W_{q - 1} \phi( \cdots \phi(W_1 x) \cdots ))
\end{align}
where we overload notation to let $\phi(\cdot)$ denote the element-wise application of the activation $\phi$. Let $m_i$ denote the size of the $i$-th hidden layer, so $W_1 \in \R^{m_1 \times d}, W_2 \in \R^{m_2 \times m_1}, \cdots, W_q \in \R^{1 \times m_{q - 1}}$. We will let $\mathcal{M}\triangleq (m_1, \ldots, m_{q-1})$ denote the sequence of hidden layer sizes. We will focus on $\ell_2$-regularized logistic loss (see~\eqref{eq:binary_logistic_loss}, using $\|\cdot\|_F$ and $r = 2$) and denote it by $L_{\lambda, \mathcal{M}}$. 

Following notation established in this section, we denote the optimizer of $L_{\lambda,\mathcal{M}}$ by $\genparam_{\lambda,\mathcal{M}}$, the normalized margin of $\genparam_{\lambda,\mathcal{M}}$ by $\gamma_{\lambda,\mathcal{M}}$, the max-margin solution by $\genparam^{\star,\mathcal{M}}$, and the max-margin by $\gamma^{\star,\mathcal{M}}$, assumed to be positive. Our notation emphasizes the architecture of the network. 

We can define the population 0-1 loss of the network parameterized by $\genparam$ by 
$
L(\genparam) \triangleq \Pr_{(x, y) \sim \pdata} [y f^{\textup{NN}}(x;\genparam) \le 0]
$. We let $\mathcal{X}$ denote the data domain and $C \triangleq \sup_{x \in \mathcal{X}} \|x\|_2$ denote the largest possible norm of a single datapoint.

By combining the neural net complexity bounds of~\citet{golowich2017size}~with our Theorem~\ref{thm:binary_margin}, we can conclude that optimizing weakly-regularized logistic loss gives generalization bounds that depend on the maximum possible network margin for the given architecture.

\begin{corollary}
	\label{cor:margingeneralization}
		Suppose $\phi$ is 1-Lipschitz and 1-positive-homogeneous. With probability at least $1 - \delta$ over the draw of $(x_1, y_1),\ldots, (x_n, y_n)$ i.i.d. from $\pdata$, we can bound the test error of the optimizer of the regularized loss by
	\begin{align}
	\label{eq:thetalambdastargen}
	\limsup_{\lambda \rightarrow 0} L(\Theta_{\lambda,\mathcal{M}}) \lesssim  \frac{C}{\gamma^{\star, \mathcal{M}}q^{\frac{q-1}{2} \sqrt{n}}} + \epsilon(\gamma^{\star,\mathcal{M}})
	\end{align}
	where $\epsilon(\gamma) \triangleq \sqrt{\frac{\log \log_2 \frac{4C}{\gamma}}{n}} + \sqrt{\frac{\log (1/\delta)}{n}}$. Note that $\epsilon(\gamma^{\star, \mathcal{M}})$ is primarily a smaller order term, so the bound mainly scales with $\frac{C}{\gamma^{\star, \mathcal{M}} q^{(q - 1)/2}\sqrt{n}}$. \footnote{Although the $\frac{1}{q^{(q - 1)/2}}$ factor of \eqref{eq:genbound} decreases with depth $q$, the margin $\gamma$ will also tend to decrease as the constraint $\|\bgenparam\|_F \le 1$ becomes more stringent.} \end{corollary}

Finally, we observe that the maximum normalized margin is non-decreasing with the size of the architecture. Formally, for two depth-$q$ architectures $\mathcal{M} = (m_1, \ldots, m_{q - 1})$ and $\mathcal{M'} = (m'_1, \ldots, m'_{q - 1})$, we say $\mathcal{M} \le \mathcal{M'}$ if $m_i \le m'_i \ \forall i = 1, \ldots q-1$. Theorem~\ref{thm:marginnondecreasing} states if $\mathcal{M} \le \mathcal{M'}$, the max-margin over networks with architecture $\mathcal{M'}$ is at least the max-margin over networks with architecture $\mathcal{M}$. 

\begin{theorem}
	\label{thm:marginnondecreasing}
	Recall that $\gamma^{\star,\mathcal{M}}$ denotes the maximum normalized margin of a network with architecture $\mathcal{M}$.
		If $\mathcal{M} \le \mathcal{M'}$, we have
	$$
	\gamma^{\star,\mathcal{M}} \le \gamma^{\star,\mathcal{M'}}
	$$
	As a important consequence, the generalization error bound of Corollary~\ref{cor:margingeneralization} for $\mathcal{M'}$ is at least as good as that for $\mathcal{M}$.
\end{theorem}
This theorem is simple to prove and follows because we can directly implement any network of architecture $\mathcal{M}$ using one of architecture $\mathcal{M'}$, if $\mathcal{M}\le \mathcal{M'}$. This highlights one of the benefits of over-parametrization: the margin does not decrease with a larger network size, and therefore Corollary~\ref{cor:margingeneralization} gives a better generalization bound. In Section~\ref{sec:additionalexp}, we provide empirical evidence that the test error decreases with larger network size while the margin is non-decreasing.

The phenomenon in Theorem~\ref{thm:marginnondecreasing} contrasts with standard $\ell_2$-normalized linear prediction. In this setting, adding more features increases the norm of the data, and therefore the generalization error bounds could also increase. On the other hand, Theorem~\ref{thm:marginnondecreasing} shows that adding more neurons (which can be viewed as learned features) can only improve the generalization of the max-margin solution.

%% file: experiments.tex
\section{Simulations}
\label{sec:experiments}
We empirically validate our theory with several simulations. First, we train a two-layer net on synthetic data with and without explicit regularization starting from the same initialization in order to demonstrate the effect of an explicit regularizer on generalization. We confirm that the regularized network does indeed generalize better and moves further from its initialization. For this experiment, we use a large initialization scale, so every weight $\sim \mathcal{N}(0,1)$. We average this experiment over 20 trials and plot the test accuracy, normalized margin, and percentage change in activation patterns in Figure~\ref{fig:explicitreg}. We compute the percentage of activation patterns changed over every possible pair of hidden unit and training example. Since a low percentage of activations change when $\lambda = 0$, the unregularized neural net learns in the kernel regime. Our simulations demonstrate that an explicit regularizer improves generalization error as well as the margin, as predicted by our theory. 

The data comes from a ground truth network with $10$ hidden networks, input dimension $20$, and a ground truth unnormalized margin of at least $0.01$. We use a training set of size $200$ and train for $20000$ steps with learning rate $0.1$, once using regularizer $\lambda = 5 \times 10^{-4}$ and once using regularization $\lambda = 0$. We note that the training error hits 0 extremely quickly (within 50 training iterations). The initial normalized margin is negative because the training error has not yet hit zero.

\begin{figure}
	\centering
	\includegraphics[width=0.3\textwidth]{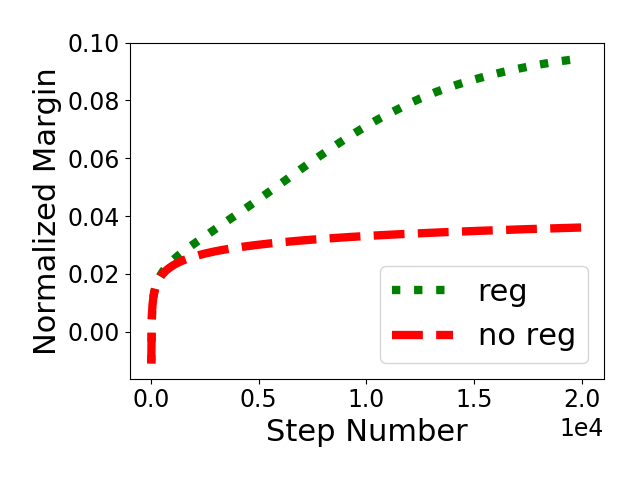}
	\includegraphics[width=0.3\textwidth]{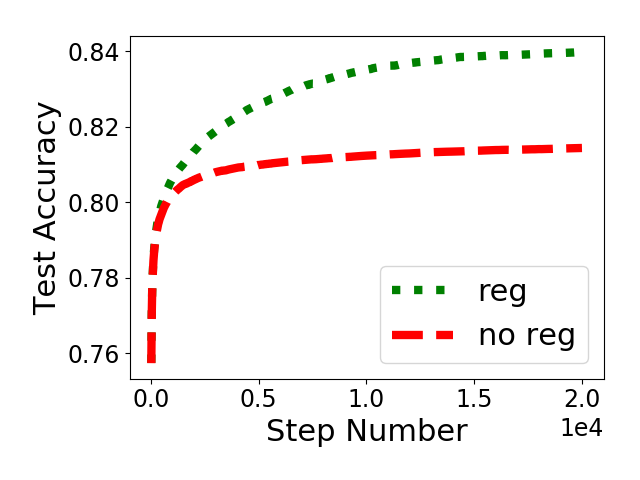}
	\includegraphics[width=0.3\textwidth]{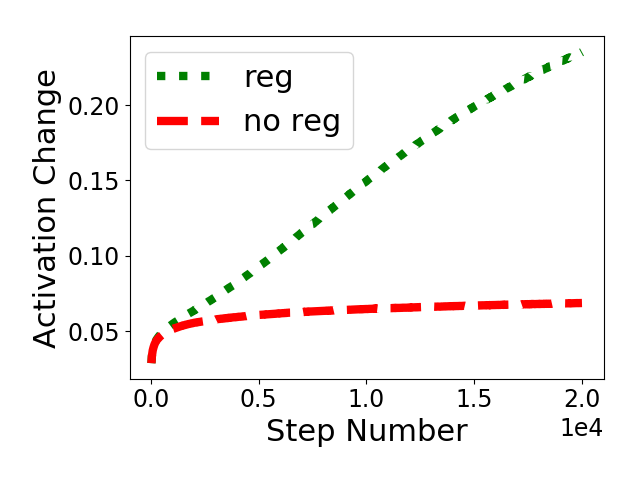}
	\caption{Comparing regularization and no regularization starting from the same initialization. \textbf{Left:} Normalized margin. \textbf{Center:} Test accuracy. \textbf{Right:} Percentage of activation patterns changed.}
	\label{fig:explicitreg}
\end{figure}

We also compare the generalization of a regularized neural net and kernel method as the sample size increases. Furthermore, we demonstrate that for two-layer nets, the test error decreases and margin increases as the width of the hidden layer grows, as predicted by our theory. We provide figures and full details in Section~\ref{sec:additionalexp}.

%% file: conclusion.tex
\section{Conclusion}

We have shown theoretically and empirically that explicitly $\ell_2$ regularized neural nets can generalize better than the corresponding kernel method. We also argue that maximizing margin is one of the inductive biases of~\relu~networks obtained from optimizing weakly-regularized cross-entropy loss. To complement these generalization results, we study optimization and prove that it is possible to find a global minimizer of the regularized loss in polynomial time when the network width is infinite. A natural direction for future work is to apply our theory to optimize the margin of finite-sized neural networks. 

\section*{Acknowledgments}
CW acknowledges the support of a NSF Graduate Research Fellowship. JDL acknowledges support of the ARO under MURI Award W911NF-11-1-0303.  This is part of the collaboration between US DOD, UK MOD and UK Engineering and Physical Research Council (EPSRC) under the Multidisciplinary University Research Initiative. We also thank Nati Srebro and Suriya Gunasekar for helpful discussions in various stages of this work.

%% file: proof_kernel_lower_bound.tex
\section{Additional Notation}
In this section we collect additional notations that will be useful for our proofs. 

Let $\sph \triangleq \{\bar{u} \in \R^d : \|\bar{u}\|_2 = 1\}$ be the unit sphere in $d$ dimensions. Let $\mathcal{L}_k^2(\sph)$ be the space of functions on $\sph \rightarrow \R^k$ for which the squared $\ell_2$ norm of the function value is Lebesgue integrable. For $\varphi_1, \varphi_2 \in \spacetwo_k(\sph)$, we can define $\langle \varphi_1, \varphi_2 \rangle \triangleq \int_{\sph} \varphi_1(\bdir)^\top \varphi_2(\bdir) d\bdir < \infty$. 

For general $p$, will also define $\mathcal{L}_1^p(\sph)$ be the space of functions on $\sph$ for which the $p$-th power of the absolute value is Lebesgue integrable. For $\varphi \in \mathcal{L}_1^p(\sph)$, we overload notation and write $\|\varphi\|_p \triangleq \left(\int_{\sph} |\varphi(\bdir)|^p d\bdir\right)^{1/p}$. Additionally, for $\varphi_1 \in \spaceone_1(\sph)$ and $\varphi_2 \in \spaceinf_1(\sph)$, we can define $\langle \varphi_1, \varphi_2 \rangle \triangleq \int_{\sph} \varphi_1(\bdir) \varphi_2(\bdir) d\bdir < \infty$. 

\section{Missing Material from Section~\ref{sec:kernel_vs_nn}} \label{sec:proof_comparison}
\subsection{Lower Bound on NTK Kernel Generalization} 
\label{subsec:kernellb}
\newcommand{\acosk}{K_1}
\newcommand{\sink}{K_2}
\newcommand{\coeff}{\beta}
In this section we will lower bound the test error of the kernel prediction function for our distribution $\cD$ in the setting of Theorem~\ref{thm:comparison}. We will first introduce some additional notation to facilitate the proofs in this section. Let $\mathcal{D}_{x}$ be the marginal distribution of $\cD$ over datapoints $x$. We use $z_i$ to refer to the last $d - 2$ coordinates of $x_i$. For a given vector $x$, $x_{-2}$ will index the last $d - 2$ coordinates of a vector $x$ and for $z \in \R^{d -2}$, use $(a, b, z)$ to denote the vector in $\R^{d}$ with first two coordinates $a, b$, and last $d -2$ coordinates $z$. For a vector $x \in \R^{d}$, let $x^{\otimes 2} \in \R^{d^2}$ denote the vector with $(i-1)d + j$-th entry $e_i^\top x e_j^\top x$. 

Furthermore, we define the following lifting functions $\varphi_{\textup{grad}}, \varphi_{\textup{relu}}$ mapping data $x \in \R^d$ to an infinite feature vector: 
\begin{align*}
\varphi_{\textup{grad}}(x) \in \mathcal{L}^2_d(\sph) &\textup{ satisfies } \varphi_{\textup{grad}}(x)[\bdir] = \one(x^\top \bdir \ge 0)x\\
\varphi_{\textup{relu}}(x) \in \mathcal{L}^\infty_1(\sph) &\textup{ satisfies } \varphi_{\textup{relu}}(x)[\bdir] = [x^\top \bdir]_+
\end{align*}
Note that the kernel $K(x', x)$ can be written as a sum of positive scalings of $\langle \varphi_{\textup{grad}}(x), \varphi_{\textup{grad}}(x')\rangle$ and $\langle \varphi_{\textup{relu}}(x), \varphi_{\textup{relu}}(x') \rangle$. We now define the following functions $\acosk, \sink : \R^{d} \times\R^{d} \mapsto \R$: 
\begin{align*}
\acosk(x', x) &= x^\top x' \left(1 - \pi^{-1} \arccos \left(\frac{x^\top x'}{\|x\|_2 \|x'\|_2}\right)\right)\\
\sink(x', x) &= \frac{\|x\|_2 \|x'\|_2}{\pi} \sqrt{1 - \left(\frac{x^\top x'}{\|x\|_2\|x'\|_2}\right)^2}
\end{align*}
We have
\begin{align*}
\begin{split}
\langle \varphi_{\textup{grad}}(x), \varphi_{\textup{grad}}(x')\rangle &= K_1(x', x)\\
\langle \varphi_{\textup{relu}}(x), \varphi_{\textup{relu}}(x') \rangle &= c_{\textup{relu}} (K_1(x', x) + K_2(x', x))
\end{split}
\end{align*}
for some $c_{\textup{relu}} > 0$. The second equation follows from Lemma A.1 of~\citep{du2017gradient}. To see the first one, we note that the indicator $\one({x'}^\top \bdir \ge 0)\one(x^\top \bdir \ge 0)$ is only 1 in a arc of degree $\pi - \arccos(x^\top x'/\|x\|_2 \|x'\|_2)$ between $x$ and $x'$. As all directions are equally likely, the expectation $\E_{\bdir}[\one({x'}^\top \bdir \ge 0)\one(x^\top \bdir \ge 0)] = 1 - \pi^{-1}\arccos\left(\frac{x^\top x'}{\|x\|_2 \|x'\|_2}\right)$. 

Then as the kernel $K(x', x)$ is the sum of positive scalings of $\langle \varphi_{\textup{grad}}(x), \varphi_{\textup{grad}}(x')\rangle$ and $\langle \varphi_{\textup{relu}}(x), \varphi_{\textup{relu}}(x') \rangle$, we can express
\begin{align}
K(x', x) = \tau_1 K_1(x', x) + \tau_2(K_1(x', x) + K_2(x', x)) \label{eq:kernel_sum}
\end{align}
for $\tau_1, \tau_2 > 0$. This decomposition will be useful in our analysis of the lower bound. The following theorem restates our lower bound on the test error of any $\ell_2$-regularized kernel method. 

\begin{theorem}\label{thm:kernellb}
	For the distribution $\mathcal{D}$ defined in Section~\ref{sec:kernel_vs_nn}, if $n \lesssim d^2$, with probability $1 - \exp(-\Omega(\sqrt{n}))$ over $(x_1, y_1),\ldots, (x_n, y_n)$ drawn i.i.d. from $\cD$, for all choices of $\beta$, in test time the kernel prediction function $f^{\textup{kernel}}(\cdot;\beta)$ will predict the sign of $y$ wrong $\Omega(1)$ fraction of the time:
	\begin{align*}
	\Pr_{(x, y) \sim \mathcal{D}}[f^{\textup{kernel}}(x;\beta) y \le 0]= \Omega(1)
	\end{align*}
\end{theorem}

As it will be clear from context, we drop the $^\textup{kernel}$ superscript. The first step of our proof will be demonstrating that the first two coordinates do not affect the value of the prediction function $f(x; \beta)$ by very much. This is where we formalize the importance of having the sign of the positive label be unaffected by the sign of the first coordinate, and likewise for the second coordinate and negative labels. We utilize the sign symmetry to induce further cancellations in the prediction function output. Formally, we will first define the functions $\tilde{\acosk}, \tilde{\sink}: \R^{d - 2} \times \R^{d - 2} \mapsto \R$ with \begin{align*}
\tilde{\acosk}(z', z) = \acosk((0, 1, z'), (1, 0, z))\\
\tilde{\sink}(z', z) = \sink((0, 1, z'), (1, 0, z))
\end{align*}
Next, we will define the function $\tilde{f} : \R^{d - 2} \mapsto \R$ with
\begin{align*}
\tilde{f}(z; \beta) = \tau_1 \sum_{i = 1}^n \beta_i \tilde{\acosk}(z_i, z) + \tau_2 \sum_{i = 1}^n \beta_i (\tilde{\acosk}(z_i, z) + \tilde{\sink}(z_i, z))
\end{align*}
The following lemma states that $2\tilde{f}(z; \beta)$ will approximate both $f((1, 0, z); \beta) + f((-1, 0, z); \beta)$ and $f((0, 1, z); \beta) + f((0, -1, z); \beta)$. This allows us to immediately lower bound the test error of $f$ by the probability that $\tilde{f}(z; \beta)$ is sufficiently large. 
\begin{lemma}
	\label{lem:f_tilde_approx}
	Define the functions 
	\begin{align*}
	f^+(z; \beta) &\triangleq f((1, 0, z); \beta) + f((-1, 0, z); \beta)\\
	f^-(z; \beta) &\triangleq f((0, 1, z); \beta) + f((0, -1, z); \beta)
	\end{align*} 
	Then with probability $1 - \exp(-\Omega(d))$, there is some universal constant $c$ such that
	\begin{align}
	\begin{split}
	|f^+(z; \beta) - 2\tilde{f}(z; \beta)| \le \frac{c (\tau_1 + \tau_2)}{d} \sum_{i = 1}^n |\beta_i|\\
	|f^-(z; \beta) - 2\tilde{f}(z; \beta)| \le \frac{c (\tau_1 + \tau_2)}{d} \sum_{i = 1}^n |\beta_i|	 \label{eq:f_tilde_approx}
	\end{split}
	\end{align}
	As a result, for all choices of $\beta_1, \ldots, \beta_n$, we can lower bound the test error of the kernel prediction function by
	\begin{align*}
	\Pr_{(x, y) \sim \mathcal{D}}[f(x;\beta) y \le 0] \ge \frac{1}{4} \Pr_{z \sim \{-1, +1\}^{d - 2}}\left(|\tilde{f}(z; \beta)| \ge \frac{3c(\tau_1 + \tau_2)}{2d} \sum_{i = 1}^n |\beta_i|\right) - \exp(-\Omega(d))
	\end{align*}
\end{lemma}
Now we argue that $|\tilde{f}(z; \beta)|$ will be large with constant probability over $z$, leading to constant test error of $f$. Formally we first show that with constant probability over the choice of $z \sim \{-1, +1\}^{d-2}$, we have $|\tilde{f}(z; \beta)| \ge \frac{3c(\tau_1 + \tau_2)}{2d} \sum_{i = 1}^n |\beta_i|$.
\begin{lemma}
	\label{lem:f_tilde_large}
	For sufficiently small $n \lesssim d^2$, with probability $1 - \exp(-\Omega(\sqrt{n}))$ over the random draws of $z_1, \ldots, z_n$, the following holds: for all $\beta_1, \ldots, \beta_n$, we will have
	\begin{align*}
 \Pr_{z \sim \{-1, +1\}^{d - 2}}\left(|\tilde{f}(z; \beta)| \ge \frac{3c(\tau_1 + \tau_2)}{2d} \sum_{i = 1}^n |\beta_i|\right) \ge \Omega(1)
	\end{align*}
	where $c$ is the constant defined in Lemma~\ref{lem:f_tilde_approx}.
\end{lemma}
This will allow us to complete the proof of Theorem~\ref{thm:kernellb}. 
\begin{proof}[Proof of Theorem~\ref{thm:kernellb}]
	By plugging Lemma~\ref{lem:f_tilde_large} into the statement of Lemma~\ref{lem:f_tilde_approx}, we can conclude that for sufficiently small $n \lesssim d^2$, with probability $1 - \exp(-\Omega(\sqrt{n}))$ over the random draws of $z_1, \ldots, z_n$, we have 
	\begin{align*}
		\Pr_{(x, y) \sim \cD} [f(x; \beta)y \le 0] \ge \Omega(1)
	\end{align*}
	for all choices of $\beta$. This gives precisely Theorem~\ref{thm:kernellb}.
\end{proof}

It now suffices to prove Lemmas~\ref{lem:f_tilde_approx} and~\ref{lem:f_tilde_large}. 

To prove Lemma~\ref{lem:f_tilde_approx}, we will rely on the following two lemmas relating $\acosk, \sink$ with $\tilde{\acosk}$, $\tilde{\sink}$, stated and proved below: 
\begin{lemma}
	\label{lem:acosapprox}
	Let $z \in \{-1, +1\}^{d-2}$ be a uniform random point from the $d-2$-dimensional hypercube and $x \in \supp(\mathcal{D}_x)$ be given. With probability $1 - \exp(-\Omega(d))$ over the choice of $z$, we have  
	\begin{align*}
	|\acosk(x, (1, 0, z)) + \acosk(x, (-1, 0, z)) - 2\tilde{\acosk}(x_{-2}, z)| &\lesssim \frac{1}{d}\\
	|\acosk(x, (0, 1, z)) + \acosk(x, (0, -1, z)) - 2\tilde{\acosk}(x_{-2}, z)| &\lesssim \frac{1}{d}
	\end{align*}
\end{lemma}
\begin{lemma}
	\label{lem:sinkapprox}
	In the same setting as Lemma~\ref{lem:acosapprox}, with probability $1 - \exp(-\Omega(d))$ over the choice of $z$, we have
	\begin{align*}
	|\sink(x, (1, 0, z)) + \sink(x, (-1, 0, z)) - 2\tilde{\sink}(x_{-2}, z)| &\lesssim \frac{1}{d}\\
	|\sink(x, (0, 1, z)) + \sink(x, (0, -1, z)) - 2\tilde{\sink}(x_{-2}, z)| &\lesssim \frac{1}{d}
	\end{align*}
\end{lemma}

\begin{proof}[Proof of Lemma~\ref{lem:acosapprox}]
	As it will be clear in the context of this proof, we use $x_1$ to denote the first coordinate of $x$ and $x_2$ to denote the second coordinate of $x$.
	We prove the first inequality, as the proof for the second is identical. First, note that if $x_1 = 0$,$|x_2| = 1$, then we have $\acosk(x, (1, 0, z)) + \acosk(x, (-1, 0, z)) = 2\acosk((0, 1, x_{-2}), (1, 0, z))$ so the inequality holds trivially. Thus, we work in the case that $|x_1| = 1$, $x_2 = 0$. 
	
	Note that $\|(1,0,z)\|_2 = \|(-1,0,z)\|_2= \|x\|_2 = \sqrt{d - 1}$. We have: 
	\begin{align}
	&\acosk(x, (1,0,z)) + \acosk(x,(-1,0,z))\\&= \left(1 - \pi^{-1}\arccos\left(\frac{1 + x_{-2}^\top z}{d-1}\right)\right)(1 + x_{-2}^\top z)\nonumber \\
	&+ \left(1 - \pi^{-1}\arccos\left(\frac{-1 +x_{-2}^\top z}{d -1}\right)\right)(-1 + x_{-2}^\top z) \nonumber\\
	&=\pi^{-1}\left(\arccos\left(\frac{-1 + x_{-2}^\top z}{d - 1}\right) - \arccos\left(\frac{1 + x_{-2}^\top z}{d -1}\right)\right) \label{eq:kernelapprox:1}\\
	&+x_{-2}^\top z\left(2-\pi^{-1}\arccos\left(\frac{-1 +x_{-2}^\top z}{d - 1}\right) - \pi^{-1}\arccos\left(\frac{1 + x_{-2}^\top z}{d -1}\right)\right)\label{eq:kernelapprox:2}
	\end{align}	
	Now we perform a Taylor expansion of $\arccos$ around $\nu \triangleq x_{-2}^\top z /(d -1)$ to get 
	$$\arccos(\nu + \epsilon) = \arccos(\nu) + \arccos'(\nu) \epsilon + O(\epsilon^2)$$
	for any $|\nu|, |\nu + \epsilon| \le 3/4$. Note that this happens with probability $1- \exp(-\Omega(d))$ by Hoeffding's inequality. Furthermore, for $|\nu| \le 3/4$, $\arccos'(\nu) = O(1)$, so we get that~\eqref{eq:kernelapprox:1} can be bounded by $O(\frac{1}{d})$. Next, we claim the following: 
	\begin{align}
	\left|\arccos\left(\frac{-1 + x_{-2}^\top z}{d -1}\right) +\arccos\left(\frac{1 + x_{-2}^\top z}{d -1}\right) - 2\arccos\left(\frac{x_{-2}^\top z}{d -1}\right)\right| = O\left(\frac{1}{d^2} \right) \nonumber
	\end{align}
	This follows simply from Taylor expansion around $\nu$ setting $\epsilon$ to $\pm\frac{1}{d-1}$. Substituting this into~\eqref{eq:kernelapprox:2}~and using our bound on~\eqref{eq:kernelapprox:1}, we get
	\begin{align}
	\left|\acosk(x, (1,0,z)) + \acosk(x,(-1,0,z)) - 2x_{-2}^\top z\left(1 - \pi^{-1}\arccos\left(\frac{x_{-2}^\top z}{d-1}\right)\right)\right| \le O\left(\frac{1}{d}\right) \nonumber
	\end{align}
	Now we use the fact that $x_{-2}^\top z\left(1 - \pi^{-1}\arccos\left(\frac{x_{-2}^\top z}{d-1}\right)\right) = \acosk((0, 1, x_{-2}), (1, 0, z))$ to complete the proof. 
\end{proof}

\begin{proof}[Proof of Lemma~\ref{lem:sinkapprox}]
	As before, it suffices to prove the first inequality in the case that $|x_1| = 1$, $x_2 = 0$. We can compute
	\begin{align}
	\begin{split}
	(\sink(x, (1,0,z)) + \sink(x, (-1,0,z)) &=\\ \frac{1}{\pi}\left((d-1) \sqrt{1 - \left( \frac{1+x_{-2}^\top z}{d-1}\right)^2} +  (d-1) \sqrt{1 - \left( \frac{-1+x_{-2}^\top z}{d-1}\right)^2}\right)\label{eq:sinkapprox:1}
	\end{split}
	\end{align}
	Now we again perform a Taylor expansion, this time of $g(v) = \sqrt{1 - v^2}$ around $\nu \triangleq \frac{x_{-2}^\top z}{d - 1}$. We get 
	\begin{align*}
	g(\nu + \epsilon) = g(\nu) + g'(\nu) \epsilon + O(\epsilon^2)
	\end{align*}
	for any $|\nu|, |\nu+ \epsilon| \le 3/4$. Note that $|\nu|,|\nu + \epsilon| \le 3/4$ with probability $1 - \exp(-\Omega(d))$ via straightforward concentration. It follows that 
	\begin{align*}
	\left|\sqrt{1 - \left(\frac{1 + x_{-2}^\top z}{d - 1}\right)^2} + \sqrt{1 - \left(\frac{-1 + x_{-2}^\top z}{d - 1}\right)^2} - 2 \sqrt{1 - \left(\frac{x_{-2}^\top z}{d - 1}\right)^2} \right| \lesssim \frac{1}{d^2}
	\end{align*}
	Now plugging this into~\eqref{eq:sinkapprox:1} and using the fact that $\frac{1}{\pi}(d - 1)\sqrt{1 - \left(\frac{x_{-2}^\top z}{d - 1}\right)^2} = \sink((0, 1, x_{-2}), (1, 0, z))$ gives the desired result.
\end{proof}

Now we can complete the proof of Lemma~\ref{lem:f_tilde_approx}. 
\begin{proof}[Proof of Lemma~\ref{lem:f_tilde_approx}]
	We note that 
	\begin{align}
	\begin{split}
	|f^+(z; \beta) - 2\tilde{f}(z; \beta)| =& \bigg|(\tau_1 + \tau_2)\sum_{i = 1}^n \beta_i [\acosk((1, 0, z), x_i) + \acosk((-1, 0, z), x_i) - 2\tilde{\acosk}(z_i,z)] \\
	&+\tau_2 \sum_{i = 1}^n  \beta_i [\sink((1, 0, z), x_i) + \sink((-1, 0, z), x_i) - 2\tilde{\sink}(z_i,z)] \bigg| \label{eq:signequalbound:1}
	\end{split}
	\end{align}		Now with applying Lemmas~\ref{lem:acosapprox} and~\ref{lem:sinkapprox} with a union bound over all $i$, we get with probability $1 - \exp(-\Omega(d))$ over the choice of $z$ uniform from $\{-1, +1\}^{d - 2}$, for all $i$
	\begin{align*}
	|\acosk((1, 0, z), x_i) + \acosk((-1, 0, z), x_i) - 2\tilde{\acosk}(z_i,z)| &\lesssim \frac{1}{d}\\
	|\sink((1, 0, z), x_i) + \sink((-1, 0, z), x_i) - 2\tilde{\sink}(z_i,z)| & \lesssim \frac{1}{d}
	\end{align*}
	Now plugging into~\eqref{eq:signequalbound:1} and applying triangle inequality gives us 
	\begin{align}
	|f^+(z; \beta) - 2\tilde{f}(z; \beta)| \le \frac{c (\tau_1 + \tau_2)}{d} \sum_{i = 1}^n |\beta_i|
	\label{eq:signequalbound:2}
	\end{align}
	with probablity $1 - \exp(-\Omega(d))$ over $z$ for some universal constant $c$. An identical argument also gives us 
	\begin{align}
	|f^-(z; \beta) - 2\tilde{f}(z; \beta)| \le \frac{c (\tau_1 + \tau_2)}{d} \sum_{i = 1}^n |\beta_i|	 \label{eq:signequalbound:3}
	\end{align}
	Finally, to lower bound the quantity $\Pr_{(x, y) \sim \mathcal{D}}[f(x;\beta) y \le 0]$, we note that if 
	\begin{align*}
	|\tilde{f}(z; \beta)| \ge \frac{3c(\tau_1 + \tau_2)}{2d} \sum_{i = 1}^n |\beta_i|
	\end{align*}
	and~\eqref{eq:f_tilde_approx} hold, then $f^+(z; \beta)$ and $f^-(z; \beta)$ will have the same sign. However, this in turn means that one of the following must hold: \begin{align*}
	f((1,0,z);\beta) < 0\\
	f((-1,0,z);\beta)< 0\\
	f((0,1,z);\beta)> 0\\
	f((0,-1,z);\beta) > 0
	\end{align*}
	which implies an incorrect predicted sign. As $(1, 0, z)$, $(-1, 0, z)$, $(0, 1, z)$, $(0, -1, z)$ are all equally likely under distribution $\cD_x$, the probability of drawing one of these examples under $\cD_x$ is at least $$\frac{1}{4} \Pr_{z \sim \{-1, +1\}^{d - 2}}\left(|\tilde{f}(z; \beta)| \ge \frac{3c(\tau_1 + \tau_2)}{2d} \sum_{i = 1}^n |\beta_i|\right) - \exp(-\Omega(d))$$ This gives the desired lower bound on $\Pr_{(x, y) \sim \mathcal{D}}[f(x;\beta) y \le 0]$. 
\end{proof}

Now we will prove Lemma~\ref{lem:f_tilde_large}. We will first construct a polynomial approximation $\hat{f}(z;\beta)$ of $\tilde{f}(z; \beta)$, and then lower bound the expectation $\E_z[\hat{f}(z; \beta)^2]$. We use the following two lemmas:
\begin{lemma} \label{lem:polyexpansion}
	Define the polynomial $g : \R \mapsto \R$ as follows:
	\begin{align*}
	g(x) \triangleq \tau_1(d - 1)\left(\frac{1}{2}x + \frac{1}{\pi}x^2 + \frac{1}{6\pi}x^4\right) + \tau_2(d - 1)\left(\frac{1}{\pi} + \frac{1}{2}x + \frac{1}{2\pi}x^2 + \frac{1}{24\pi}x^4 \right)
	\end{align*} 
	Then for $z \in\{-1, +1\}^{d-2}$ distributed uniformly over the hypercube and some given $z' \in \{-1, +1\}^{d-2}$, 
	\begin{align*}
	\Pr_z\left[\left|g\left(\frac{z^\top z'}{d-1}\right) - (\tau_1 + \tau_2)\tilde{\acosk}(z, z')- \tau_2 \tilde{\sink}(z, z')\right| \le c_1(\tau_1 + \tau_2)\frac{\log^{2.5}}{d^{1.5}}\right] \ge 1 - d^{-10}
	\end{align*}
	for some universal constant $c_1$.
\end{lemma}
\begin{lemma}
	\label{lem:polyvarlb}
	Let $g : \R \mapsto \R$ be any degree-$k$ polynomial with nonnegative coefficients, i.e. $g(x) = \sum_{j = 1}^k a_j x^j$ with $a_j \ge 0$ for all $j$. For $n \lesssim d^2$, with probability $1 - \exp(-\Omega(\sqrt{n}))$ over the random draws of $z_1, \ldots, z_n$ i.i.d. uniform from $\{-1, +1\}^d$, the following holds: for all $\beta_1, \ldots, \beta_n$, we will have
	\begin{align*}
	\E_z \left[\left(\sum_{i = 1}^n \beta_i g(z^\top z_i)\right)^2\right] \gtrsim a_2^2 d^2 \sum_{i = 1}^n {\beta_i}^2
	\end{align*}
	where $z \in \{-1, +1\}^d$ is a uniform vector from the hypercube.
\end{lemma}

Now we provide the proof of Lemma~\ref{lem:f_tilde_large}. 
\begin{proof}[Proof of Lemma~\ref{lem:f_tilde_large}]
	For the degree-4 polynomial $g$ defined in Lemma~\ref{lem:polyexpansion}, we define 
	\begin{align*}
	\hat{f}(z; \beta) = \sum_{i = 1}^n \beta_i g\left(\frac{z^\top z_i}{d - 1}\right)
	\end{align*}
	Note that with probability $1 - d^{-8}$ over the choice of $z$, $|\hat{f}(z; \beta) - \tilde{f}(z; \beta)| \lesssim \frac{\log^{2.5} d}{d^{1.5}}(\tau_1+\tau_2) \sum_{i = 1}^n |\beta_i|$.
	
	With the purpose of applying Lemma~\ref{lem:polyvarlb}, we can first compute the coefficent of $x^2$ in $g(x/(d- 1))$ to be $\frac{1}{\pi(d - 1)} \left(\tau_1 + \tau_2/2\right)$. As $g$ has positive coefficients, we can thus apply Lemma~\ref{lem:polyvarlb} to conclude that with high probability over $z_1, \ldots,z_n$, the following event $\mathcal{E}$ holds: for all choices of $\beta_1, \ldots, \beta_n$, $\E_z[\hat{f}(z; \beta)^2] \ge c_2(\tau_1 + \tau_2)^2 \sum_{i = 1}^n {\beta_i}^2$ for some universal constant $c_2$. We now condition on the event that $\mathcal{E}$ holds. 
	
	Note that by Cauchy-Schartz, $\sum_{i = 1}^n {\beta_i}^2 \ge \frac{1}{n} (\sum_{i = 1}^n |\beta_i|)^2$. It follows that if $n \le \frac{c_2}{4c^2}d^2$, we have
	\begin{align*}
	\E_z[\hat{f}(z; \beta)^2] \ge c_2(\tau_1 + \tau_2)^2 \sum_{i = 1}^n \beta_i^2 \ge \frac{c_2 (\tau_1 + \tau_2)^2}{n}(\sum_{i = 1}^n |\beta_i|)^2 \ge \frac{4c^2 (\tau_1 + \tau_2)^2}{d^2}(\sum_{i = 1}^n |\beta_i|)^2 
	\end{align*}
	Now we can apply Bonami's Lemma (see Chapter 9 of~\citet{o2014analysis}) along with the fact that $\hat{f}$ is a degree-4 polynomial in i.i.d. $\pm 1$ variables $z_1, \ldots, z_{d - 2}$ to obtain
	\begin{align*}
	\E_z[\hat{f}(z; \beta)^4] \le 9^4 (\E_z[\hat{f}(z; \beta)^2])^2
	\end{align*}
	Combining this with Proposition 9.4 of~\citet{o2014analysis} lets us conclude that if $\mathcal{E}$ holds, with probability $\Omega(1)$ over the random draw of $z$, 
	\begin{align*}
	|\hat{f}(z; \beta)| \ge \frac{3}{4} \sqrt{\E_{z} [\hat{f}(z; \beta)^2]} \ge \frac{3c(\tau_1+\tau_2)}{2d}\sum_{i= 1}^n |\beta_i|
	\end{align*}
	Since $|\hat{f}(z; \beta) - \tilde{f}(z; \beta)| \lesssim \frac{(\tau_1 + \tau_2)\log^{2.5}(d)}{d^{1.5}} \sum_{i = 1}^n |\beta_i|$ w.h.p over $z$, we can conclude that $$|\tilde{f}(z; \beta)| \ge\frac{3c(\tau_1+\tau_2)}{2d}\sum_{i= 1}^n |\beta_i|$$ holds with probability $\Omega(1)$ over $z$. This gives the desired result.
	\end{proof}

\begin{proof}[Proof of Lemma~\ref{lem:polyexpansion}]
	Define functions $h_1, h_2 :(-1,1) \mapsto \R$ with 
	\begin{align*}
	h_1(x) &= x(1 - \pi^{-1} \arccos x)\\
	h_2(x) &= \frac{1}{\pi}\sqrt{1 - x^2}
	\end{align*}
	Recalling our definitions of $\tilde{\acosk}$, $\tilde{\sink}$, it follows that $\tilde{\acosk}(z, z') = (d -1)h_1\left(\frac{z^\top z'}{d-1}\right)$ and $\tilde{\sink}(z,z') = (d-1)h_2\left(\frac{z^\top z'}{d-1}\right)$. Letting $g_1$, $g_2$ denote the 4-th order Taylor expansions around 0 of $h_1$, $h_2$, respectively, it follows from straightforward calculation that
	\begin{align*}
	g_1(x) &= \frac{1}{2}x + \frac{1}{\pi}x^2 + \frac{1}{6\pi}x^4 \\
	g_2(x) &= \frac{1}{\pi}- \frac{1}{2\pi}x^2- \frac{1}{8\pi}x^4 
	\end{align*}
	with $|h_1(x) - g_1(x)|\le O(|x|^5)$ and $|h_2(x) - g_2(x)| \le O(|x|^5)$ for $|x| \le 3/4$. )Now we can observe that $g(x) = (\tau_1 + \tau_2)(d-1) g_1(x) + \tau_2(d-1)g_2(x)$. Thus,
	\begin{align*}
	g\left(\frac{z^\top z'}{d-1}\right)- (\tau_1 + \tau_2)\tilde{\acosk}(z, z')- \tau_2 \tilde{\sink}(z, z')  \\ = (d-1)\left[(\tau_1+\tau_2)\left(g_1\left(\frac{z^\top z'}{d-1}\right) - h_1\left(\frac{z^\top z'}{d-1}\right)\right) + \tau_2\left(g_2\left(\frac{z^\top z'}{d-1}\right) - h_2\left(\frac{z^\top z'}{d-1}\right)\right) \right]\label{eq:polyexpansion:1}
	\end{align*}
	As $|z^\top z'|/(d-1)\le3/4$ with probability $1- \exp(-\Omega(d))$, the above is bounded in absolute value by $(d- 1)(\tau_1 +\tau_2) O\left(\left(\frac{|z^\top z'|}{d-1}\right)^5\right)$. Finally, by Hoeffding's inequality $|z^\top z'|\le c\sqrt{d\log d}$ with probability $1 - d^{-10}$ for some universal constant $c$. This gives the desired bound.
\end{proof}

\begin{proof}[Proof of Lemma~\ref{lem:polyvarlb}]
	We first compute 
	\begin{align}
	\E_z \left[\left(\sum_{i = 1}^n \beta_i g(z^\top z_i)\right)^2\right] &= 	\E_z \left[\left(\sum_{i = 1}^n \beta_i \sum_{j=1}^k a_j (z^\top z_i)^j \right)^2\right] \nonumber\\
	&= 	\E_z \left[\left( \sum_{j=1}^k a_j \sum_{i = 1}^n \beta_i (z^\top z_i)^j \right)^2\right] \nonumber\\
	&= \sum_{j_1, j_2}a_{j_1}a_{j_2}\E_z \left[\left(\sum_{i = 1}^n \beta_i (z^\top z_i)^{j_1} \right)\left(\sum_{i = 1}^n \beta_i (z^\top z_i)^{j_2} \right)\right] \tag{expanding the square and using linearity of expectation}
	\end{align}
	Now note that all terms in the above sum are nonnegative by Lemma~\ref{lem:cube_exp} and the fact that $a_{j_1}, a_{j_2} \ge 0$. Thus, we can lower bound the above by the term corresponding to $j_1 = j_2 = 2$: 
	\begin{align*}
	\E_z \left[\left(\sum_{i = 1}^n \beta_i g(z^\top z_i)\right)^2\right] \ge a_2^2 \E_z \left[\left(\sum_{i = 1}^n \beta_i (z^\top z_i)^{2} \right)\left(\sum_{i = 1}^n \beta_i (z^\top z_i)^{2} \right)\right]
	\end{align*} 
	Now we can express
	\begin{align}
	\E_z \left[\left(\sum_{i = 1}^n \beta_i (z^\top z_i)^{2} \right)\left(\sum_{i = 1}^n \beta_i (z^\top z_i)^{2} \right)\right] = \beta^\top {M^{\otimes 2}}^\top  \E_z[z^{\otimes 2}{z^{\otimes 2}}^\top] M^{\otimes 2} \beta \label{eq:polyvarlb:2}
	\end{align}
	where $M \in \R^{d \times n}$ is the matrix with $z_i$ as its columns, and $M^{\otimes 2}$ has $z_i^{\otimes 2}$ as its columns. 
	
	We first compute $\E_z [z^{\otimes 2} {z^{\otimes 2}}^\top]$. Note that the entry in the $d(i_1-1) + j_1$-th row and $d(i_2-1) +j_2$-th column of $z^{\otimes 2} {z^{\otimes 2}}^\top$ is given by $(e_{i_1}^\top z) (e_{j_1}^\top z)(e_{i_2}^\top z) (e_{j_2}^\top z)$. Note that unless $i_1 = i_2$, $j_1 = j_2$ or $i_1 = j_1$, $i_2 = j_2$, this value has expectation 0. Thus, $\E_z [z^{\otimes 2} {z^{\otimes 2}}^\top]$ is a matrix with 1 on its diagonals and entries in the $(i-1)d + i$-th row and $(j-1)d + j$-th column, and 0 everywhere else. Letting $S$ denote the set of indices $\{(i - 1)d + i : i \in [d]\}$ and $\ones_S$ denote the vector in $\R^{d^2}$ with ones on $S$ and 0 everywhere else, we thus have  
	\begin{align*}
	\E_z [z^{\otimes 2} {z^{\otimes 2}}^\top] = \ones_S \ones_S^\top + I_{[d^2] \setminus S \times [d^2] \setminus S}
	\end{align*}
	Now letting $M^{\otimes 2}_S$ denote $M^{\otimes 2}$ with rows whose indices are not in $S$ zero'ed out, it follows that 
	\begin{align}
	{M^{\otimes 2}}^\top  \E_z[z^{\otimes 2}{z^{\otimes 2}}^\top] M^{\otimes 2} &= {M_S^{\otimes 2}}^\top \ones_S \ones_S^\top M_S^{\otimes 2} + {M_{[d^2] \setminus S}^{\otimes 2}}^\top I_{[d^2] \setminus S \times [d^2] \setminus S}M_{[d^2] \setminus S}^{\otimes 2}
	\nonumber\\
	&\succeq {M_{[d^2] \setminus S}^{\otimes 2}}^\top M_{[d^2] \setminus S}^{\otimes 2} \label{eq:polyvarlb:1}
	\end{align}
	Therefore, it suffices to show $\sigma_{\min}({M_{[d^2] \setminus S}^{\otimes 2}}^\top M_{[d^2] \setminus S}^{\otimes 2}) \gtrsim d^2$ with high probability. To do this, we can simply invoke Proposition 7.9 of~\citet{soltanolkotabi2019theoretical} using $\eta_{\min} = \eta_{\max} = \sqrt{d^2 - d}$ and the fact that the columns of $M_{[d^2] \setminus S}^{\otimes 2}$ are $O(1)$-sub-exponential (Claim~\ref{claim:subexp} to get that if $n \le c d^2$ for some universal constant $c$, then $\sigma_{\min}^2(M_{[d^2] \setminus S}^{\otimes 2}) \gtrsim d^2$ with probability $1- \exp(O(\sqrt{n}))$.
	
	Finally, combining this with~\eqref{eq:polyvarlb:1} and~\eqref{eq:polyvarlb:2} gives the desired result. 
\end{proof}

\begin{claim}\label{claim:subexp}
	Say that a random vector $x \in \R^d$ is $B$-sub-exponential if the following holds: 
	\begin{align*}
	\sup_{y \in \sph} \inf\{C > 0 : \E\exp(|x^\top y|/C)\le 2\} \le B
	\end{align*}
	Suppose that $z \sim\{-1, +1\}^d$ is a uniform vector on the hypercube. Then there is a universal constant $c$ such that $z^{\otimes{2}} - \ones_S$ is  $c$-sub-exponential, where $S \triangleq \{(i - 1)d + i : i \in [d]\}$ is the set of indices corresponding to squared entries of $z^{\otimes{2}}$.   
\end{claim}
\begin{proof}
	Let $\tilde{z}^{\otimes 2}$ denote the $d^2 - d$ dimensional vector which removes coordinates in $S$ from $z^{\otimes 2}$. As $z^{\otimes 2}$ has value $1$ with probability $1$ on coordinates in $S$, it suffices to show that $\tilde{z}^{\otimes 2}$ is $c$-sub-exponential. We first note that for any $y \in \R^{d^2 - d}$, $y^\top \tilde{z}^{\otimes 2}$ can be written as $z^\top Y z$, where $Y$ is a $d\times d$ matrix with $0$ on its diagonals and $ij$-th entry matching the corresponding entry of $y$. 
	
	Now we can apply Theorem 1.1 of~\citet{rudelson2013hanson}, using the fact that $e_i^\top z$ have sub-Gaussian norm $2$ to get 
	\begin{align*}
	\Pr[|z^\top Y z| > t] \le 2\exp(-c' t^2/16\|y\|_2^2)
	\end{align*}
	for some universal constant $c'$. Since this holds for all $y$, we can conclude the claim statement using Lemma 5.5 of~\citet{soltanolkotabi2019theoretical}.
\end{proof}

The following lemma is useful for proving the lower bound in Lemma~\ref{lem:polyvarlb}.
\begin{lemma}\label{lem:cube_exp}
	Let $z_i \in \{-1, +1\}^d$ for $i \in [n]$, and let $z \in \{-1, +1\}^d$ be a vector sampled uniformly from the hypercube. Then for any integers $p, q \ge 0$, 
	\begin{align*}
	\E_z \left[\left(\sum_i \beta_i (z^\top z_i)^q \right)\left(\sum_i \beta_i (z^\top z_i)^p\right)\right] \ge 0
	\end{align*}
	Furthermore, equality holds if exactly one of $p$ or $q$ is odd. 
\end{lemma}

In order to prove Lemma~\ref{lem:cube_exp}, we will require some tools and notation from boolean function analysis (see~\citet{o2014analysis} for a more in-depth coverage). We first introduce the following notation: for $x \in \{-1, +1\}^d$ and $S \subseteq [d]$, we use $x^S$ to denote $\prod_{s \in S} x_s$. Then by Theorem 1.1 of~\cite{o2014analysis}, we can expand a function $f : \{-1, +1\}^d \mapsto \R$ with respect to the values $x^S$:
\begin{align*}
f(x) = \sum_{S \subseteq [d]} \hat{f}(S) x^S
\end{align*}
where $\hat{f}(S)$ is called the Fourier coefficient of $f$ on $S$ and $\hat{f}(S) = \E_{x}[f(x) x^S]$ for $x$ uniform on $\{-1, +1\}^d$. For functions $f_1, f_2 : \{-1, +1\}^d \mapsto \R$, the following identity holds:
\begin{align}\label{eq:fourierinnerprod}
\E_x[f_1(x) f_2(x)] = \sum_{S \subseteq [d]} \hat{f}_1(S) \hat{f}_2(S)
\end{align}

\begin{proof}[Proof of Lemma~\ref{lem:cube_exp}]
	For this proof we will use double indices on the $z_i$ vectors, so that $z_{i,j}$ will denote the $j$-th coordinate of $z_i$. We will only use the symbols $j$ to index the vectors $z, z_1,\ldots, z_n$. We define the functions $g(z) \triangleq \sum_i \beta_i (z^\top z_i)^q$ and $h(z) \triangleq \sum_i \beta_i (z^\top z_i)^p$, with Fourier coefficients $\hat{g}$, $\hat{h}$, respectively, and $g_i(z) = (z^\top z_i)^q$, $h_i(z) = (z^\top z_i)^p$ with Fourier coefficients $\hat{g}_i$, $\hat{h}_i$. We claim that for any $S \subseteq [d]$, $\hat{g}(S)\hat{h}(S) \ge 0$. 
	
	To see this, we will first compute $\hat{g}_i(S)$ as follows: $\hat{g}_i(S) = \E_z [(z^\top z_i)^q z^S]$. Now note that if we expand $(z^\top z_i)^q$ and compute this expectation, only terms of the form $z^S z_i^S z_{j_1}^{a_1} \cdots z_{j_k}^{a_k} z_{i, j_1}^{a_1} \cdots z_{i, j_k}^{a_k} $ with $a_1, \ldots, a_k$ even and $a_1 + \cdots + a_k = q - |S|$ are nonzero. Note that we have allowed $k$ to vary. Thus,
	\begin{align}
	\E_z [(z^\top z_i)^q z^S] &= \sum_{j_1, \ldots j_k, a_1, \ldots, a_k}(z^S)^2 z_i^S z_{j_1}^{a_1} \cdots z_{j_k}^{a_k} z_{i, j_1}^{a_1} \cdots z_{i, j_k}^{a_k} \nonumber\\
	&= c_{q,|S|} z_i^S \label{eq:cube_exp:1}
	\end{align}
	for some positive integer $c_{q, |S|}$ depending only on $q, |S|$. We obtained~\eqref{eq:cube_exp:1} via symmetry and the fact that $(z^S)^2=1$, $z_{j_1}^{a_1} \cdots z_{j_k}^{a_k} z_{i, j_1}^{a_1} \cdots z_{i, j_k}^{a_k} = 1$, as they are squares of values in $\{-1, +1\}$. Note that $c_{q,|S|} = 0$ for $|S| > q$. It follows that $\hat{g}(S)= c_{q,|S|} \sum_i \beta_i z_i^S$, and $\hat{h}(S) = c_{p, |S|} \sum_i \beta_i z^S_i$. Thus, $\hat{g}(S) \hat{h}(S) \ge 0 \forall S$, which means by~\eqref{eq:fourierinnerprod}, we get
	\begin{align*}
	\E_z [g(z) h(z)] = \sum_S \hat{g}(S) \hat{h}(S) \ge 0
	\end{align*}
	as desired. 
	
	Now to see that $\E_z[g(z) h(z)] = 0$ if exactly one of $p$ or $q$ is odd, note that every monomial in the expansion of $g(z) h(z)$ will have odd degree. However, the expectation of such monomials is always $0$ as $z \in \{-1, +1\}^d$. 
\end{proof}

\subsection{Proof of Theorem~\ref{thm:comparison}}

We now complete the proof of Theorem~\ref{thm:comparison}. Note that the kernel lower bound follows from~\ref{thm:kernellb}, so it suffices to upper bound the generalization error of the neural net solution. 
\begin{proof}[Proof of Theorem~\ref{thm:comparison}]
	We first invoke Theorem~\ref{thm:multi_opterrrequired}~to conclude that with $\lambda=\poly(n)^{-1}$, the network $f^{\textup{NN}}(\cdot; \Theta_{\lambda})$ will have margin that is a constant factor approximation to the max-margin. 
		
	For neural nets with at least 4 hidden units, we now construct a neural net with a good normalized margin: 
	 $$f^{\textup{NN}}(x) = [x^\top e_1]_+  + [-x^\top e_1]_+ -[x^\top e_2]_+ - [-x^\top e_2]_+ $$
	 As this network has constant norm and margin 1, it has normalized margin $\Theta(1)$, and therefore the max neural net margin is $\Omega(1)$. Now we apply the generalization bound of Proposition~\ref{prop:genbound}~to obtain
	 \begin{align*}
	 \Pr_{x,y \sim \cD}[f^{\textup{NN}}(x; \Theta_{\lambda})y \le 0]\lesssim \sqrt{\frac{d}{n}} + \sqrt{\frac{\log \log (16d)}{n}} + \sqrt{\frac{\log (1/\delta)}{n}}
	 \end{align*}
	 as desired. Choosing $\delta = n^{-5}$ gives the desired result. Combined with the Theorem~\ref{thm:kernellb}~lower bound on the kernel method, this completes the proof. 
\end{proof}

\subsection{Regression Setting}\label{subsec:regression_compare}
In this section we argue that a analogue to Theorem~\ref{thm:comparison}~holds in the regression setting where we test on a truncated squared loss $\ell(\hat{y};y) = \min((y - \hat{y})^2, 1)$. As the gap exists for the same distribution $\cD$, the theorem statement is essentially identical to the classification setting, and the kernel lower bound carries over. For the regularized neural net upper bound, we will only highlight the differences here. 

\begin{theorem}
	\label{thm:nnl1normregression}
	Let $f^{\textup{NN}}(\cdot;\genparam)$ be some two-layer neural network 	with $m$ hidden units parametrized by $\genparam$, as in Section~\ref{sec:kernel_vs_nn}. Define the $\lambda$-regularized squared error loss
	\begin{align*}
	L_{\lambda,m}(\genparam) \triangleq \frac{1}{n}\sum_{i=1}^n (f^{\textup{NN}}(x_i;\genparam)-y_i)^2 + \lambda \|\genparam\|_2^2
	\end{align*}
	with $\genparam_{\lambda, m}\in \argmin_{\genparam} L_{\lambda,m}(\genparam)$. Suppose there exists a width-$m$ network that fits the data $(x_i, y_i)$ perfectly. Then as $\lambda \rightarrow 0$, $L_{\lambda, m}(\genparam_{\lambda, m}) \rightarrow 0$ and $\|\genparam_{\lambda, m}\|_2 \rightarrow \|\Theta^{\star,m}\|_2^2$, where $\Theta^{\star,m}$ is an optimizer of the following problem:
	\begin{align}
	\begin{split}
		\min_{\Theta}& \|\Theta\|_2^2 \\
		\textup{such that }& f^{\textup{NN}}(x_i; \Theta) = y_i \ \forall i
		\end{split}\label{eq:minnormnetreg}
	\end{align} 
\end{theorem}
\begin{proof}
We note that $\lambda \|\genparam_{\lambda, m}\|_2^2 \le L_{\lambda, m}(\genparam_{\lambda, m})\le L_{\lambda, m}(\genparam^{\star,m}) =\lambda \|\genparam^{\star,m}\|_2^2$, so as $\lambda \rightarrow 0$, and also $\|\genparam_{\lambda, m}\|_2 \le \|\genparam^{\star, m}\|_2$. Now assume for the sake of contradiction that $\exists B$ with $\|\genparam_{\lambda, m}\|_2\le B < \|\genparam^{\star, m}\|_2$ for arbitrarily small $\lambda$. We define 
	\begin{align*}
	r^{\star} \triangleq &\min_{\genparam} \frac{1}{n} \sum_{i= 1}^n (f^{\textup{NN}}(
	x_i;\genparam)- y_i)^2 \\
	& \subj \|\genparam\|_2 \le B
	\end{align*}
	Note that $r^{\star} > 0$ since $\genparam^{\star, m}$ is optimal for \eqref{eq:minnormnetreg}. However, $L_{\lambda, m} \ge r^{\star}$ for arbitrarily small $\lambda$, a contradiction. Thus, $\lim_{\lambda \rightarrow 0} \|\genparam_{\lambda,m}\|_2^2 =  \|\genparam^{\star,m}\|_2^2$. 
\end{proof}

For the distribution $\cD$, the neural net from the proof of Theorem~\ref{thm:comparison}~also fits the data perfectly in the regression setting. As this network has norm $O(1)$, we can apply the norm-based Rademacher complexity bounds of~\citet{golowich2017size}~in the same manner as in Section~\ref{sec:rademacher}~(using standard tools for Lipschitz and bounded functions) to conclude a generalization error bound of $\tilde{O}\left(\sqrt{\frac{d \log n + \log(1/\delta)}{n}}\right)$, same as the classification upper bound.

\subsection{Connection to the $\ell_1$-SVM}\label{subsec:ell1svm}
In this section, we state a known connection between a $\ell_2$ regularized two-layer neural net and the $\ell_1$-SVM over relu features~\citep{neyshabur2014search}. Following our notation from Section~\ref{sec:margin}, we will use $\gamma^{\star,m}$ to denote the maximum possible normalized margin of a two-layer network with hidden layer size $m$ (note the emphasis on the size of the single hidden layer). 

The depth $q=2$ case of Corollary~\ref{cor:margingeneralization} implies that optimizing weakly-regularized $\ell_2$ loss over width-$m$ two-layer networks gives parameters whose generalization bounds depend on the hidden layer size only through $1/\gamma^{\star,m}$. Furthermore, from Theorem~\ref{thm:marginnondecreasing} it immediately follows that 
$
\gamma^{\star,1} \le \gamma^{\star,2} \le \cdots \le \gamma^{\star,\infty}
$.
The work of \citet{neyshabur2014search} links $\gamma^{\star,m}$ to the $\ell_1$ SVM over the lifted features $\varphi_{\textup{relu}}$. We look at the margin of linear functionals corresponding to $\mu\in \mathcal{L}^1_1(\sph)$. The 1-norm SVM \citep{zhu20041} over the lifted feature $\varphi_{\textup{relu}}(x)$ solves for the maximum margin:
\begin{align}
\begin{split}
\gamma_{\ell_1}\triangleq& \max_{\mu} \min_{i\in [n]} y_i\langle \mu, \varphi_{\textup{relu}}(x_i)\rangle \\
& \subj  \|\mu\|_1 \le 1
\label{eq:l1svm}
\end{split}
\end{align}
This formulation is equivalent to a hard-margin optimization on ``convex neural networks'' \citep{bengio2006convex}. \citet{bach2017breaking} also study optimization and generalization of convex neural networks. Using results from~\citep{rosset2007l1, neyshabur2014search, bengio2006convex}, our Theorem~\ref{thm:multilabelmargin} implies that optimizing weakly-regularized logistic loss over two-layer networks is equivalent to solving \eqref{eq:l1svm} when the size of the hidden layer is at least $n + 1$. Proposition~\ref{prop:gammastar=gammal1} states this deduction.\footnote{The factor of $\frac12$ is due the the relation that every unit-norm parameter $\Theta$ corresponds to an $\mu$ in the lifted space with $\|\mu\|=2$.}
\begin{proposition}
	\label{prop:gammastar=gammal1}
	Let $\gamma_{\ell_1}$ be defined in~\eqref{eq:l1svm}. If margin $\gamma_{\ell_1}$ is attainable by some solution $\mu \in \spaceone_1(\sph)$, then
	$
	\frac{\gamma_{\ell_1}}{2} = \gamma^{\star,n+1} = \cdots = \gamma^{\star,\infty}
	$.
\end{proposition}

%% file: proofmargin.tex
\section{Missing Material for Section~\ref{sec:margin}} \label{sec:proof:margin}

\subsection{Multi-class Setting} \label{subsec:multiclass_margin}
We will first state our analogue of Theorem~\ref{thm:binary_margin}~in the multi-class setting, as the proofs for the binary case will follow by reduction to the multi-class case.

In the same setting as Section~\ref{sec:margin}, let $l$ be the number of multi-class labels, so the $i$-th example has label $y_i \in [l]$. Our family $\mathcal{F}$ of prediction functions $f$ now takes outputs in $\R^l$, and we now study the $\lambda$-regularized cross entropy loss, defined as
\begin{align}
L_{\lambda}(\genparam) \triangleq -\frac{1}{n}\sum_{i = 1}^n \log \frac{\exp(f_{y_i}(x_i; \genparam))}{\sum_{j= 1}^l \exp(f_j(x_i;\genparam))} + \lambda \|\genparam\|^r \label{eq:multilabelcrossentropy}
\end{align}
We redefine the normalized margin of $\genparam_{\lambda}$ as:
\begin{align}
\gamma_{\lambda} \triangleq \min_{i} (f_{y_i}(x_i;\bgenparam_{\lambda}) - \max_{j \ne y_i}f_j(x_i; \bgenparam_{\lambda}))
\label{eq:gammalambda}
\end{align}
Define the $\|\cdot \|$-max normalized margin as  $$\gamma^\star \triangleq \max_{\|\genparam\| \le 1}  [ \min_{i} (f_{y_i}(x_i; \genparam) - \max_{j \ne y_i}f_j(x_i; \genparam))]$$ and  let $\genparam^\star$ be a parameter achieving this maximum. With these new definitions, our theorem statement for the multi-class setting is identical as the binary setting:
\begin{theorem}
	\label{thm:multilabelmargin}
	Assume $\gamma^\star > 0$ in the multi-class setting with cross entropy loss. Then as $\lambda \rightarrow 0$, $\gamma_{\lambda} \rightarrow \gamma^\star$. 
\end{theorem}
Since $L_{\lambda}$ is typically hard to optimize exactly for neural nets, we study how accurately we need to optimize $L_{\lambda}$ to obtain a margin that approximates $\gamma^\star$ up to a constant. We show that for $\lambda$ polynomial in $n, \gamma^\star$, and $l$, it suffices to find $\genparam'$ achieving a constant factor $\alpha$ multiplicative approximation of $L_{\lambda}(\genparam_{\lambda})$ in order to have margin $\gamma'$ satisfying $\gamma' \ge \frac{\gamma^\star}{\alpha^{a/r}}$.
\begin{theorem}
	\label{thm:multi_opterrrequired}
	In the setting of Theorem~\ref{thm:multilabelmargin}, suppose that we choose 
	$
	\lambda = \exp(-(2^{r/a} - 1)^{-a/r})\frac{(\gamma^\star)^{r/a}}{n^c(l - 1)^c}
	$ for sufficiently large $c$ (that only depends on $r/a$). For $\alpha \le 2$, let $\genparam'$ denote a $\alpha$-approximate minimizer of $L_{\lambda}$, so $L_{\lambda}(\genparam') \le \alpha L_{\lambda}(\genparam_{\lambda})$. Denote the normalized margin of $\genparam'$ by $\gamma'$. Then
	$\gamma' \ge \frac{\gamma^\star}{10 \cdot \alpha^{a/r}}.$
	\end{theorem}

Towards proving Theorem~\ref{thm:multilabelmargin}, we first prove that $L_{\lambda}$ does indeed have a global minimizer. 
\begin{claim}
	\label{claim:Lhasminimizer}
	In the setting of Theorems~\ref{thm:multilabelmargin} and~\ref{thm:binary_margin}, $\argmin_{\genparam} L_{\lambda}(\genparam)$ exists.
\end{claim}
\begin{proof}
	We will argue in the setting of Theorem~\ref{thm:multilabelmargin} where $L_{\lambda}$ is the multi-class cross entropy loss, because the logistic loss case is analogous. We first note that $L_{\lambda}$ is continuous in $\genparam$ because $f$ is continuous in $\genparam$ and the term inside the logarithm is always positive. Next, define $b \triangleq \inf_{\genparam} L_{\lambda}(\genparam) > 0$. Then we note that for $\|\genparam\| > (b/\lambda)^{1/r} \triangleq M$, we must have $L_{\lambda}(\genparam) > b$. It follows that $\inf_{\|\genparam\| \le M} L_{\lambda}(\genparam) =  \inf_{\genparam} L_{\lambda}(\genparam) $. However, there must be a value $\genparam_{\lambda}$ which attains $\inf_{\|\genparam\| \le M} L_{\lambda}(\genparam)$, because $\{\genparam: \|\genparam\| \le M\}$ is a compact set and $L_{\lambda}$ is continuous. Thus, $\inf_{\genparam} L_{\lambda}(\genparam)$ is attained by some $\genparam_{\lambda}$. 
\end{proof}

Next we present the following lemma, which says that as we decrease $\lambda$, the norm of the solution $\|\Theta_{\lambda}\|$ grows. 
\begin{lemma} \label{lem:thetalambdagrowsmulti}
	In the setting of Theorem~\ref{thm:multilabelmargin}, as $\lambda \rightarrow 0$, we have $\|\genparam_{\lambda}\| \rightarrow \infty$.
\end{lemma}

To prove Theorem \ref{thm:multilabelmargin}, we rely on the exponential scaling of the cross entropy: $L_{\lambda}$ can be lower bounded roughly by $\exp(-\|\genparam_{\lambda}\| \gamma_{\lambda})$, but also has an upper bound that scales with $\exp(-\|\genparam_{\lambda}\| \gamma^{\star})$. By Lemma \ref{lem:thetalambdagrowsmulti}, we can take large $\|\genparam_{\lambda}\|$ so the gap $\gamma^{\star} - \gamma_{\lambda}$ vanishes. This proof technique is inspired by that of \citet{rosset2004boosting}.

\begin{proof}[Proof of Theorem \ref{thm:multilabelmargin}]
	For any $M > 0$ and $\Theta$ with $\gamma_{\Theta} \triangleq \min_{i} \left(f(x_i;\bgenparam) - \max_{j \ne y_i}f(x_i;\bgenparam)\right)$,
	\begin{align}
	L_{\lambda}(M\Theta)& = \frac{1}{n}\sum_{i = 1}^n -\log \frac{\exp(M^a f_{y_i}(x_i;\genparam))}{\sum_{j = 1}^l \exp(M^a f_j(x_i;\genparam))} + \lambda M^r\|\Theta\|^r \tag{by the homogeneity of $f$}\\ 
	&= \frac{1}{n}\sum_{i = 1}^n -\log \frac{1}{1 + \sum_{j \ne y_i} \exp(M^a( f_j(x_i;\genparam) - f_{y_i}(x_i;\genparam)))} + \lambda M^r\|\Theta\|^r \label{eq:losshomogeneity}\\
	& \le \log(1 + (l - 1)\exp(-M^a \gamma_{\Theta})) + \lambda M^r\|\genparam\|^r \label{eq:loss_scale}
	\end{align}
	We can also apply $\sum_{j \ne y_i} \exp(M^a( f_j(x_i;\genparam) - f_{y_i}(x_i;\genparam))) \ge \max \exp(M^a( f_j(x_i;\genparam) - f_{y_i}(x_i;\genparam))) = 
	\exp \gamma_\Theta$ in order to lower bound \eqref{eq:losshomogeneity} and obtain
	\begin{align}
	L_{\lambda}(M\Theta)
	& \ge \frac 1 n \log(1 + \exp(-M^a \gamma_{\Theta})) + \lambda M^r\|\genparam\|^r \label{eq:loss_scale_LB}
	\end{align}
	
	Applying~\eqref{eq:loss_scale} with $M = \|\Theta_\lambda\|$ and $\Theta = \Theta^\star$, noting that $\|\Theta^\star\| \le 1$, we have: 
	\begin{align}
	L_{\lambda}(\genparam^\star\|\genparam_{\lambda}\|) 
	&\le \log(1 + (l - 1)\exp(-\|\genparam_{\lambda}\|^a \gamma^\star)) + \lambda \|\genparam_{\lambda}\|^r \label{eq:ubLmarginmulti}
	\end{align}
	Next we lower bound $L_{\lambda}(\genparam_{\lambda})$ by applying~\eqref{eq:loss_scale_LB}, 
	\begin{align}
	L_{\lambda}(\genparam_{\lambda}) 
	&\ge \frac{1}{n}\log(1 + \exp(-\|\genparam_{\lambda}\|^a \gamma_{\lambda})) + \lambda \|\genparam_{\lambda}\|^r \label{eq:lbLmarginmulti}
	\end{align}
	Combining \eqref{eq:ubLmarginmulti} and \eqref{eq:lbLmarginmulti} with the fact that $L_{\lambda}(\genparam_{\lambda}) \le L_{\lambda}(\genparam^\star \|\genparam_{\lambda}\|)$ (by the global optimality of $\Theta_\lambda$), we have 
	$$\forall \lambda > 0, n \log(1 + (l - 1)\exp(-\|\genparam_{\lambda}\|^a \gamma^\star)) \ge  \log(1 + \exp(-\|\genparam_{\lambda}\|^a \gamma_{\lambda}))$$
	
	Recall that by Lemma~\ref{lem:thetalambdagrowsmulti}, as $\lambda\rightarrow 0$, we have $\|\Theta_\lambda\|\rightarrow \infty$. Therefore, $\exp(-\|\genparam_{\lambda}\|^a \gamma^\star), \exp(-\|\genparam_{\lambda}\|^a \gamma_{\lambda})\rightarrow 0$. Thus, we can apply Taylor expansion to the equation above with respect to $\exp(-\|\genparam_{\lambda}\|^a \gamma^\star)$ and $\exp(-\|\genparam_{\lambda}\|^a \gamma_{\lambda})$. If $\max\{\exp(-\|\genparam_{\lambda}\|^a \gamma^\star), \exp(-\|\genparam_{\lambda}\|^a \gamma_{\lambda})\} < 1$, then we obtain
	\begin{align}
	n(l-1) \exp(-\|\genparam_{\lambda}\|^a \gamma^\star)  \ge \exp(-\|\genparam_{\lambda}\|^a \gamma_{\lambda}) - O(\max\{\exp(-\|\genparam_{\lambda}\|^a \gamma^\star)^2, \exp(-\|\genparam_{\lambda}\|^a \gamma_{\lambda})^2\})\nonumber\end{align}
	
	We claim this implies that $\gamma^\star \le \lim\inf _{\lambda \rightarrow 0} \gamma_\lambda$. If not, we have $ \lim\inf_{\lambda \rightarrow 0} \gamma_\lambda < \gamma^\star$ , which implies that the equation above is violated with sufficiently large $\|\Theta_\lambda\|$ ($\|\Theta_\lambda\| \gg \log(2(\ell-1)n)^{1/a}$ would suffice). By Lemma~\ref{lem:thetalambdagrowsmulti}, $\|\Theta_\lambda\|\rightarrow \infty$ as $\lambda \rightarrow 0$ and therefore we get a contradiction. 
	
	Finally, we have $\gamma_{\lambda} \le \gamma^\star$ by definition of $\gamma^\star$. Hence, $\lim_{\lambda \rightarrow 0} \gamma_{\lambda}$ exists and equals $\gamma^\star$. 
\end{proof}

Now we fill in the proof of Lemma~\ref{lem:thetalambdagrowsmulti}.
\begin{proof}[Proof of Lemma~\ref{lem:thetalambdagrowsmulti}]
	For the sake of contradiction, we assume that $\exists C > 0$ such that for any $\lambda_0 > 0$, there exists $0 < \lambda <\lambda_0$ with $\|\Theta_\lambda\| \le C$. 
	We will determine the choice of $\lambda_0$ later and pick $\lambda$ such that  $\|\Theta_\lambda\|\le C$. Then the logits (the prediction $f_j(x_i;\Theta)$ before softmax) are bounded in absolute value by some constant (that depends on $C$), and therefore the loss function $-\log \frac{\exp(f_{y_i}(x_i;\genparam))}{\sum_{j = 1}^l \exp(f_j(x_i;\genparam))}$for every example is bounded from below by some constant $D> 0$ (depending on $C$ but not $\lambda$.)

	Let $M = \lambda^{-1/(r+1)}$, we have that 
	\begin{align}
	0< D\le L_{\lambda}(\genparam_{\lambda}) &\le L_{\lambda}(M\genparam^\star) \tag{by the optimality of $\genparam_{\lambda}$} \\
	&\le -\log \frac{1}{1 + (l - 1)\exp(-M^a \gamma^\star)} + \lambda M^r \tag{by~\eqref{eq:loss_scale}}  \nonumber\\ 
	& = \log (1 + (l - 1) \exp(-\lambda^{-a/(r+1)} \gamma^\star)) + \lambda^{1/(r+1)} \nonumber\\
	& \le \log (1 + (l - 1) \exp(-\lambda_0^{-a/(r+1)} \gamma^\star)) + \lambda_0^{1/(r+1)} \nonumber
		\end{align}
	Taking a sufficiently small $\lambda_0$, we obtain a contradiction and complete the proof. 
\end{proof}

\subsection{Missing Proof for Optimization Accuracy}
\label{subsec:opterrproof}

\begin{proof}[Proof of Theorem~\ref{thm:multi_opterrrequired}]
	Choose $B \triangleq \left(\frac{1}{\gamma^\star} \log \frac{(l - 1)(\gamma^\star)^{r/a}}{\lambda}\right)^{1/a}$. We can upper bound $L_{\lambda}(\genparam')$ by computing 
	\begin{align*}
	L_{\lambda}(\genparam') \le \alpha L{\lambda}(\genparam_{\lambda}) &\le \alpha L_{\lambda}(B \genparam^\star)\nonumber \\
	&\le \alpha \log(1 + (l - 1)\exp(-B^a \gamma^\star)) + \alpha\lambda B^r \tag{by \eqref{eq:loss_scale}} \\
	&\le \alpha (l - 1)\exp(-B^a \gamma^\star) + \alpha\lambda B^r & \tag{using $\log(1 + x) \le x$} \nonumber \\
	&\le \alpha\frac{\lambda}{(\gamma^{\star})^{r/a}} + \alpha\lambda \left(\frac{1}{\gamma^\star} \log \frac{(l - 1)(\gamma^\star)^{r/a}}{\lambda}\right)^{r/a}\\
	&\le \alpha\frac{\lambda}{(\gamma^\star)^{r/a}} \left(1 + \left(\log \frac{(l - 1)(\gamma^\star)^{r/a}}{\lambda}\right)^{r/a}\right) \triangleq L^{(UB)} 	\end{align*}
	Furthermore, it holds that $\|\genparam'\|^r \le \frac{L^{(UB)}}{\lambda}$. Now we note that 
	\begin{align*}
	L_{\lambda}(\genparam') \le L^{(UB)} \le 2\alpha \frac{\lambda}{(\gamma^\star)^{r/a}}\left(\log \frac{(l - 1)(\gamma^\star)^{r/a}}{\lambda}\right)^{r/a} \le \frac{1}{2n}
	\end{align*} 
	for sufficiently large $c$ depending only on $a/r$. Now using the fact that $\log (x) \ge \frac{x}{1+ x} \ \forall x \ge -1$, we additionally have the lower bound $L_{\lambda}(\genparam') \ge \frac{1}{n}\log(1+ \exp(-\gamma' \|\genparam'\|^a)) \ge \frac{1}{n}\frac{\exp(-\gamma' \|\genparam'\|^a)}{1 + \exp(-\gamma' \|\genparam'\|^a)}$. Since $L^{(UB)} \le 1$, we can rearrange to get 
	\begin{align*}
	\gamma' \ge \frac{-\log \frac{nL_{\lambda}(\genparam')}{1 - nL_{\lambda}(\genparam')}}{\|\genparam'\|^a} \ge \frac{-\log \frac{nL^{(UB)}}{1 - nL^{(UB)}}}{\|\genparam'\|^a} \ge \frac{-\log (2nL^{(UB)})}{\|\genparam'\|^a}
	\end{align*}
	The middle inequality followed because $\frac{x}{1 - x}$ is increasing in $x$ for $0 \le x < 1$, and the last because $L^{(UB)} \le \frac{1}{2n}$. Since $-\log2nL^{(UB)}> 0$ we can also apply the bound $\|\genparam'\|^r \le \frac{L^{(UB)}}{\lambda}$ to get 
	\begin{align}
	\gamma' &\ge \frac{-\lambda^{a/r} \log 2 nL^{(UB)}}{(L^{(UB)})^{a/r}} \nonumber \\
	&= \frac{-\log\left(2n\alpha \frac{\lambda}{(\gamma^\star)^{r/a}} \left(1 + \left(\log \frac{(l - 1)(\gamma^\star)^{r/a}}{\lambda}\right)^{r/a}\right)\right)}{\frac{\alpha^{a/r}}{\gamma^\star}  \left(1 + \left(\log \frac{(l - 1)(\gamma^\star)^{r/a}}{\lambda}\right)^{r/a}\right)^{a/r}} \tag{by definition of $L^{(UB)}$} \\
	&\ge \frac{\gamma^\star}{\alpha^{a/r}}\left(\underbrace{\frac{\log(\frac{(\gamma^\star)^{r/a}}{2\alpha n\lambda})}{  \left(1 + \left(\log \frac{(l - 1)(\gamma^\star)^{r/a}}{\lambda}\right)^{r/a}\right)^{a/r}}}_{\clubsuit} - \underbrace{\frac{\log  \left(1 + \left(\log \frac{(l - 1)(\gamma^\star)^{r/a}}{\lambda}\right)^{r/a}\right)}{  \left(1 + \left(\log \frac{(l - 1)(\gamma^\star)^{r/a}}{\lambda}\right)^{r/a}\right)^{a/r}}}_{\heartsuit}\right) 
	\nonumber
	\end{align}
	We will first bound $\clubsuit$. First note that 
	\begin{align}
		\frac{\log(\frac{(\gamma^\star)^{r/a}}{2\alpha n\lambda})}{\log \frac{(l - 1)(\gamma^\star)^{r/a}}{\lambda}} = \frac{\log \frac{(\gamma^\star)^{r/a}}{\lambda} - \log 2\alpha n}{\log \frac{(\gamma^\star)^{r/a}}{\lambda} + \log(l - 1)} \ge \frac{\log \frac{(\gamma^\star)^{r/a}}{\lambda} - \log 2 \alpha n (l - 1)}{\log \frac{(\gamma^\star)^{r/a}}{\lambda}} \ge \frac{c - 3}{c} \label{eq:clubbound1}
	\end{align}
	where the last inequality follows from the fact that $\frac{(\gamma^\star)^{r/a}}{\lambda} \ge n^c (l - 1)^c$ and $\alpha \le 2$. Next, using the fact that $\log \frac{(\gamma^\star)^{r/a}}{\lambda} \ge \frac{1}{(2^{r/a} - 1)^{a/r}}$, we note that 
	\begin{align}
	\left(1 + \left(\log \frac{(l - 1)(\gamma^\star)^{r/a}}{\lambda}\right)^{-r/a}\right)^{a/r} \le \left(1 + \left(\frac{1}{(2^{r/a} - 1)^{a/r}}\right)^{-r/a}\right)^{a/r} \le 2
	\label{eq:clubbound2}
	\end{align}
	Combining \eqref{eq:clubbound1} and \eqref{eq:clubbound2}, we can conclude that 
	\begin{align*}
	\clubsuit = \frac{\log(\frac{(\gamma^\star)^{r/a}}{2\alpha n\lambda})}{\log \frac{(l - 1)(\gamma^\star)^{r/a}}{\lambda}}\left(1 + \left(\log \frac{(l - 1)(\gamma^\star)^{r/a}}{\lambda}\right)^{-r/a}\right)^{-a/r} \ge \frac{c - 3}{2c} 
	\end{align*}
	Finally, we note that if $1 + \left(\log \frac{(l - 1)(\gamma^\star)^{r/a}}{\lambda}\right)^{r/a}$ is a sufficiently large constant that depends only on $a/r$ (which can be achieved by choosing $c$ sufficiently large)
		it will follow that $\heartsuit \le \frac{1}{10}$. Thus, for sufficiently large $c \ge 5$, we can combine our bounds on $\clubsuit$ and $\heartsuit$ to get that 
	\begin{align*}
		\gamma' \ge \frac{\gamma^\star}{10\alpha^{a/r}}
	\end{align*}
\end{proof}

\subsection{Proofs of Theorem~\ref{thm:binary_margin}}
\label{subsec:binarymarginproof}
For completeness, we will now prove Theorem~\ref{thm:binary_margin}~via reduction to the multi-class cases. Recall that we now fit binary labels $y_i \in \{-1, +1\}$ (as opposed to indices in $[l]$) and redefine $f(\cdot;\genparam)$ to assign a single real-valued score (as opposed to a score for each label). We also work with the simpler logistic loss in~\eqref{eq:binary_logistic_loss}. 

\begin{proof}[Proof of Theorem~\ref{thm:binary_margin}]
	We prove this theorem via reduction to the multi-class case with $l = 2$. Construct $\tilde{f} : \R^d \rightarrow \R^2$ with $\tilde{f}_1(x_i;\genparam) = -\frac{1}{2} f(x_i;\genparam)$ and $\tilde{f}_2(x_i;\genparam) = \frac{1}{2} f(x_i;\genparam)$. Define new labels $\tilde{y_i} = 1$ if $y_i = -1$ and $\tilde{y_i} = 2$ if $y_i = 1$. Now note that $\tilde{f}_{\tilde{y_i}}(x_i;\genparam) - \tilde{f}_{j \ne \tilde{y_i}}(x_i;\genparam) = y_i f(x_i;\genparam)$, so the multi-class margin for $\genparam$ under $\tilde{f}$ is the same as binary margin for $\genparam$ under $f$. Furthermore, defining 
	\begin{align*}
	\tilde{L}_{\lambda}(\genparam) \triangleq \frac{1}{n}\sum_{i = 1}^n -\log \frac{\exp(\tilde{f}_{\tilde{y_i}}(x_i;\genparam))}{\sum_{j = 1}^2 \exp(\tilde{f}_j(x_i;\genparam))} + \lambda \|\genparam\|^r
	\end{align*}
	we get that $\tilde{L}_{\lambda}(\genparam) = L_{\lambda}(\genparam)$, and in particular, $\tilde{L}_{\lambda}$ and $L_{\lambda}$ have the same set of minimizers. Therefore we can apply Theorem \ref{thm:multilabelmargin} for the multi-class setting and conclude $\gamma_{\lambda} \rightarrow \gamma^\star$ in the binary classification setting.
\end{proof}

%% file: rademacher.tex
\section{Generalization Bounds for Neural Nets}
\label{sec:rademacher}

In this section we present generalization bounds in terms of the normalized margin and complete the proof of Corollary~\ref{cor:margingeneralization}. We first state the following Proposition \ref{prop:genbound}, which shows that the generalization error only depends on the parameters through the inverse of the margin on the training data. We obtain Proposition~\ref{prop:genbound} by applying Theorem 1 of~\citet{golowich2017size} with the standard technique of using margin loss to bound classification error. There exist other generalization bounds which depend on the margin and some normalization \citep{neyshabur2015norm,neyshabur2017pac,bartlett2017spectrally,neyshabur2018towards}; we choose the bounds of \citet{golowich2017size} because they fit well with $\ell_2$ normalization. 
\begin{proposition}\label{prop:genbound}[{Straightforward consequence of \citet[Theorem 1]{golowich2017size}}]
	Suppose $\phi$ is 1-Lipschitz and $1$-positive-homogeneous. With probability at least $1 - \delta$ over the draw of $X, Y$, for all depth-$q$ networks $f^{\textup{NN}}(\cdot;\genparam)$ separating the data with normalized margin $\gamma \triangleq \min_i y_i f^{\textup{NN}}(x_i;\genparam/\|\genparam\|_F) > 0$,
	\begin{align}
	\label{eq:genbound}
	L(\genparam) \lesssim \frac{C}{\gamma q^{(q -1)/2} \sqrt{n}} + \epsilon(\gamma)
	\end{align}
	where $\epsilon(\gamma) \triangleq \sqrt{\frac{\log \log_2 \frac{4C}{\gamma}}{n}} + \sqrt{\frac{\log (1/\delta)}{n}}$ and $C = \max_{x \in\mathcal{X}} \|x\|_2$ is the max norm of the data. Note that $\epsilon(\gamma)$ is typically small, and thus the above bound mainly scales with $\frac{C}{\gamma q^{(q - 1)/2}\sqrt{n}}$. \footnote{Although the $\frac{1}{K^{(K - 1)/2}}$ factor of \eqref{eq:genbound} decreases with depth $K$, the margin $\gamma$ will also tend to decrease as the constraint $\|\bgenparam\|_F \le 1$ becomes more stringent.} 
\end{proposition}
We note that Proposition~\ref{prop:genbound} is stated directly in terms of the normalized margin in order to maintain consistency in our notation, whereas prior works state their results using a ratio between unnormalized margin and norms of the weight matrices~\citep{bartlett2017spectrally}. We provide the proof in the following section.

\subsection{Proof of Proposition~\ref{prop:genbound}}
\label{subsec:proofnnbound}

We prove the generalization error bounds stated in Proposition \ref{prop:genbound}~via Rademacher complexity and margin theory. 

Assume that our data $X, Y$ are drawn i.i.d. from ground truth distribution $\pdata$ supported on $\mathcal{X} \times \mathcal{Y}$. For some hypothesis class $\hyp$ of real-valued functions, we define the empirical Rademacher complexity $\erad(\hyp)$ as follows: 
\begin{align*}
\erad(\hyp) \triangleq \frac{1}{n} \E_{\epsilon_i}\left[\sup_{f \in\hyp} \sum_{i = 1}^n \epsilon_i f(x_i)\right]
\end{align*}
where $\epsilon_i$ are independent Rademacher random variables. 
For a classifier $f$, following the notation of Section~\ref{subsec:neural_net_gen} we will use $L(f) \triangleq  \Pr_{(x, y) \sim \pdata} (y f(x) \le 0)$ to denote the population 0-1 loss of the classifier $f$. The following classical theorem \citep{koltchinskii2002empirical}, \citep{kakade2009complexity} bounds generalization error in terms of the Rademacher complexity and margin loss.
\begin{theorem} [Theorem 2 of \citet{kakade2009complexity}]
	\label{thm:rad+margin}
	Let $(x_i, y_i)_{i = 1}^n$ be drawn iid from $\pdata$. We work in the binary classification setting, so $\mathcal{Y} = \{-1, 1\}$. Assume that for all $f \in \hyp$, we have $\sup_{x \in \mathcal{X}} |f(x)| \le C$. Then with probability at least $1 - \delta$ over the random draws of the data, for every $\gamma > 0$ and $f \in\hyp$, 
	\begin{align*}
	L(f) \le \frac{1}{n} \sum_{i = 1}^n \one(y_i f(x_i) < \gamma)+ \frac{4\erad(\hyp)}{\gamma} + \sqrt{\frac{\log \log_2 \frac{4C}{\gamma}}{n}} + \sqrt{\frac{\log(1/\delta)}{2n}}
	\end{align*}
\end{theorem}

We will prove Proposition~\ref{prop:genbound} by applying the Rademacher complexity bounds of \citet{golowich2017size} with Theorem~\ref{thm:rad+margin}. 

First, we show the following lemma bounding the generalization of neural networks whose weight matrices have bounded Frobenius norms. For this proof we drop the superscript $^\textup{NN}$ as it is clear from context.
\begin{lemma}
\label{lem:generrperf}
Define the hypothesis class $\hyp_q$ over depth-$q$ neural networks by 
\begin{align*}
\hyp_q = \left\{f(\cdot;\genparam) : \|W_j\|_F \le \frac{1}{\sqrt{q}} \ \forall j\right\}
\end{align*}
Let $C \triangleq \sup_{x \in\mathcal{X}} \|x\|_2$. Recall that $L(\genparam)$ denotes the 0-1 population loss $L(f(\cdot;\genparam))$. Then for any $f(\cdot;\genparam) \in \hyp_q$ classifying the training data correctly with unnormalized margin $\gamma_{\genparam} \triangleq \min_i y_i f(x_i;\genparam) > 0$, with probability at least $1 - \delta$, 
\begin{align}
L(\genparam) \lesssim \frac{C}{\gamma_{\genparam} q^{(q - 1)/2}\sqrt{n}} + \sqrt{\frac{\log \log_2 \frac{4C}{\gamma_{\genparam}}}{n}} + \sqrt{\frac{\log(1/\delta)}{n}}
\end{align}
Note the dependence on the unnormalized margin rather than the normalized margin.
\end{lemma}
\begin{proof}
	We first claim that $\sup_{f(\cdot;\genparam) \in\mathcal{F}_q} \sup_{x \in\mathcal{X}} f(x;\genparam) \le C$. To see this, for any $f(\cdot;\genparam) \in \hyp_q$, 
	\begin{align}
		f(x;\genparam) &= W_q\phi(\cdots \phi(W_1 x) \cdots) \notag\\
		&\le \|W_q\|_F \|\phi(W_{q - 1} \phi(\cdots \phi(W_1 x) \cdots)\|_2 \notag \\
		&\le \|W_q\|_F \|W_{q - 1} \phi(\cdots \phi(W_1x)\cdots)\|_2 \tag{since $\phi$ is 1-Lipschitz and $\phi(0) = 0$, so $\phi$ performs a contraction}\\
		&< \|x\|_2 \le C \tag{repeatedly applying this argument and using $\|W_j\|_F < 1$}
	\end{align}
	Furthermore, by Theorem 1 of \citet{golowich2017size}, $\erad(\hyp_q)$ has upper bound 
	\begin{align*}
		\erad(\hyp_q) \lesssim \frac{C}{q^{(q - 1)/2}\sqrt{n}}
	\end{align*}
	Thus, we can apply Theorem~\ref{thm:rad+margin} to conclude that for all $f(\cdot;\genparam) \in \hyp_q$ and all $\gamma > 0$, with probability $1 - \delta$, 
	\begin{align*}
		L(\genparam) \lesssim \frac{1}{n} \sum_{i = 1}^n \one (y_i f(x_i;\genparam) < \gamma) + \frac{C}{\gamma q^{(q - 1)/2}\sqrt{n}} + \sqrt{\frac{\log \log_2 \frac{4C}{\gamma}}{n}} + \sqrt{\frac{\log(1/\delta)}{n}}
	\end{align*}
	In particular, by definition choosing $\gamma = \gamma_{\genparam}$ makes the first term on the LHS vanish and gives the statement of the lemma.
\end{proof}
\begin{proof}[Proof of Proposition~\ref{prop:genbound}]
	Given parameters $\genparam = (W_1, \ldots, W_q)$, we first construct parameters $\tilde{\genparam} = (\tilde{W}_1, \ldots, \tilde{W}_q)$ such that $f(\cdot;\bgenparam)$ and $f(\cdot;\tilde{\genparam})$ compute the same function, and $\|\tilde{W}_1\|_F^2 = \|\tilde{W}_2\|_F^2 = \cdots = \|\tilde{W}_q\|_F^2 \le \frac{1}{q}$. To do this, we set 
	\begin{align*}
		\tilde{W}_j = \frac{(\prod_{k = 1}^q \|W_k\|_F)^{1/k}}{\|W_j\|_F \|\genparam\|_F} W_j
	\end{align*} 
	By construction
	\begin{align}
		\|\tilde{W}_j\|_F^2 &= \frac{(\prod_{k = 1}^q \|W_k\|_F^2)^{1/k}}{\|\genparam\|_F^2} \notag\\
		&= \frac{(\prod_{k = 1}^q \|W_k\|_F^2)^{1/k}}{\sum_{k = 1}^q \|W_k\|_F^2}\notag\\
		&\le \frac{1}{k} \tag{by the AM-GM inequality}
	\end{align}
	Furthermore, we also have 
	\begin{align}
		f(x;\tilde{\genparam}) &= \tilde{W}_q\phi(\cdots \phi(\tilde{W}_1 x) \cdots ) \notag\\
		&= \prod_{j = 1}^q  \frac{(\prod_{k = 1}^q \|W_k\|_F)^{1/k}}{\|W_j\|_F \|\genparam\|_F} W_q \phi(\cdots \phi(W_1 x) \cdots ) \tag{by the homogeneity of $\phi$}\\
		&= \frac{1}{\|\genparam\|_F^q} f(x;\genparam) \notag\\
		&= f\left(x;\frac{\genparam}{\|\genparam\|_F}\right) \tag{since $f$ is $q$-homogeneous in $\genparam$}\\
		&= f(x;\bgenparam) \notag
	\end{align}
	Now we note that by construction, $L(\genparam) = L(\tilde{\genparam})$. Now $f(\cdot;\tilde{\genparam})$ must also classify the training data perfectly, has unnormalized margin $\gamma$, and furthermore $f(\cdot;\tilde{\genparam}) \in \hyp_q$. As a result, Lemma~\ref{lem:generrperf} allows us to conclude the desired statement.
\end{proof}
To conclude Corollary~\ref{cor:margingeneralization}, we apply the above on $\genparam_{\lambda, \mathcal{M}}$ and use Theorem~\ref{thm:binary_margin}.
\begin{proof}[Proof of Corollary~\ref{cor:margingeneralization}]
	Applying the statement of Proposition~\ref{prop:genbound}, with probability $1 - \delta$, for all $\lambda > 0$,
	\begin{align*}
		L(\genparam_{\lambda, \mathcal{M}}) \lesssim \frac{C}{\gamma_{\lambda, \mathcal{M}} q^{(q - 1)/2}\sqrt{n}} + \epsilon(\gamma_{\lambda, \mathcal{M}})	
	\end{align*}
	Now we take the $\limsup$ of both sides as $\lambda \rightarrow 0$: 
	\begin{align}
		\limsup_{\lambda \rightarrow 0} L(\genparam_{\lambda, \mathcal{M}}) &\lesssim \limsup_{\lambda \rightarrow 0} \frac{C}{\gamma_{\lambda, \mathcal{M}} q^{(q - 1)/2}\sqrt{n}} + \epsilon(\gamma_{\lambda, \mathcal{M}}) \nonumber	\\
		&\lesssim \frac{C}{\gamma^{\star, \mathcal{M}} q^{(q - 1)/2}\sqrt{n}} + \epsilon(\gamma^{\star, \mathcal{M}}) \tag{by Theorem~\ref{thm:binary_margin}}
	\end{align}
\end{proof}

%% file: proofopt.tex
\section{Missing Proofs in Section~\ref{sec:optimization}}
\subsection{Detailed Setup}
\label{subsec:optsetup}
We first write our regularity assumptions on $\Phi$, $R$, and $V$ in more detail:
\begin{assumption}[Regularity conditions on $\Phi$, $R$, $V$]
	\label{assu:R}
	$R$ is convex, nonnegative, Lipschitz, and smooth: $\exists M_R, C_R$ such that $\|\nabla^2 R\|_{op} \le C_R$, and $\|\nabla R\|_2 \le M_R$. 
\end{assumption}
\begin{assumption}
	\label{assu:Phi}
	$\Phi$ is differentiable, bounded and Lipschitz on the sphere: $\exists B_{\Phi}, M_{\Phi}$ such that $\|\Phi(\bparam)\| \le B_{\Phi} \ \forall \bparam \in \sphd$, and $|\Phi_i(\bparam) - \Phi_i(\bparam')| \le M_{\Phi}\|\bparam - \bparam'\|_2 \ \forall \bparam, \bparam' \in \sphd$.
\end{assumption}

\begin{assumption}
	\label{assu:V}
	$V$ is Lipschitz and upper and lower bounded on the sphere: $\exists b_V, B_V, M_V$ such that $0 < b_V \le V(\bparam) \le B_V \ \forall \bparam \in \sphd$, and $\|\nabla V(\bparam)\|_2 \le M_V \ \forall \bparam \in \sphd$. 
\end{assumption}

We state the version of Theorem~\ref{thm:polytimeopt-noparam} that collects these parameters:

\begin{theorem} [Theorem~\ref{thm:polytimeopt-noparam} with problem parameters]
	\label{thm:polytimeopt}
	Suppose that $\Phi$ and $V$ are 2-homogeneous and Assumptions \ref{assu:R}, \ref{assu:Phi}, and \ref{assu:V} hold. Fix a desired error threshold $\epsilon > 0$. Suppose that from a starting distribution $\dist_0$, a solution to the dynamics in \eqref{eq:noisywassersteindynamics} exists. Choose
	\begin{align*}
	\sigma &\triangleq \exp(-d \log(1/\epsilon)\poly(k, M_V, M_R, M_{\Phi}, b_V, B_V, C_R, B_{\Phi}, L[\dist_0] - L^{\star})) \\
	t_{\epsilon} &\triangleq \frac{d^2}{\epsilon^4}\poly(\log(1/\epsilon), k, M_V, M_R, M_{\Phi}, b_V, B_V, C_R, B_{\Phi}, L[\dist_0] - L^{\star})
	\end{align*}
	Then it must hold that $\min_{0 \le t \le t_{\epsilon}} L[\dist_t] - \inf_{\dist} L[\dist]\le 2\epsilon$. 
\end{theorem}

\subsection{Proof Outline of Theorem~\ref{thm:polytimeopt}}
\label{subsec:optoutline}
In this section, we will provide an outline of the proof of Theorem~\ref{thm:polytimeopt}. We will fill in the missing details in Section~\ref{subsec:optproof}. 

Throughout the proof, it will be useful to keep track of $W_t \triangleq \sqrt{\E_{\theta \sim \dist_t} [\|\theta\|_2^2]}$, which measures the second moment of $\dist_t$. For convenience, we will also define the constant $B_L \triangleq M_R B_{\Phi} + B_V$. The following lemma first states that this second moment will never become too large.
	
\begin{lemma}
	\label{lem:Wbound}
	Choose any $t \le \sigma B_L/b_V$. For all $0 \le t' \le t$, $W_{t'}^2 \le \frac{L[\dist_0] + \sigma tB_L}{b_V - t \sigma B_L}$. In particular, for all $t \le t_{\epsilon}$, we have $W_{t} \le W_{\epsilon}$, where $W_{\epsilon}$ is defined as follows: 
	\begin{align}
	W_{\epsilon} \triangleq \sqrt{\frac{L[\dist_0] + \sigma t_{\epsilon} B_L}{b_v - t_{\epsilon} \sigma B_L}}\label{eq:w_eps}
	\end{align}
\end{lemma}

Next, we will prove the following statement, which intuitively says that for an arbitrary choice of $\bparam \in \sphd$, if $L'[\dist_t](\bparam)$ changes by a large amount between time steps $t$ and $t + l$, the objective function must also have decreased a lot. 
\begin{lemma}
	\label{lem:fprimechange}
	Define the quantity $Q(t) \triangleq \int \Phi d\dist_t$. For every $\bparam \in \sphd$ and $0 \le t \le t + l \le t_{\epsilon}$, $\exists c_1 \triangleq \poly(k, C_R, B_{\Phi}, M_{\Phi}, B_L)$ such that 
	\begin{align}
	&|L'[\dist_t](\bparam)- L'[\dist_{t + l}](\bparam)| \le C_R B_{\Phi} \int_{t}^{t + l} \|Q'(t)\|_1\label{eq:Lprimechangeub}\\
	&\le \sigma l c_1 (W_{\epsilon}^2 + 1) + c_1 W_{\epsilon} \sqrt{l} (L[\dist_t] - L[\dist_{t+ l}] + \sigma l c_1 (W_{\epsilon}^2 + 1))^{1/2}
	\label{eq:intQnormub}
	\end{align}
	where $W_{\epsilon}$ is defined as in~\eqref{eq:w_eps}. 
\end{lemma}
The proof of Lemma~\ref{lem:fprimechange} intuitively holds because in order for $L'[\dist_t](\bparam)$ to change by a large amount, the gradient flow dynamics must have shifted $\dist_t$ by some amount, which would have resulted in some decrease of the objective $L[\dist_t]$. We will rely on the 2-homogeneity of $\Phi$ to formalize this argument. 

Next, we will rely on the convexity of $L$: letting $\dist^\star$ be an $\epsilon$-approximate global optimizer of $L$, since $L$ is convex in $\dist$, we have 
\begin{align*}
L[\dist^{\star}] \ge L[\dist_{t}] + \E_{\param \sim \dist^{\star}}[L'[\dist_{t}](\param)] - \E_{\param \sim \dist_t}[L'[\dist_t](\param)]
\end{align*}
Thus, if $\dist_t$ is far from optimality, it follows that either 1) the quantity $\E_{\param \sim \dist_t}[L'[\dist_t](\param)]$ has a large positive value or 2) there exists some descent direction $\bparam \in \sphd$ for which $L'[\dist_{t}](\param) \ll 0$. 

For the first case, we have the following guarantee that the objective decreases by a large amount: 
\begin{lemma}
	\label{lem:posdirdecrease}
	For any time $t$ with $0 \le t \le t_{\epsilon}$, we have 
	\begin{align}
	\label{eq:posdirdecrease}
	\frac{d}{dt} L[\dist_t] \le \sigma B_L(W_{\epsilon}^2 + 1) - \frac{\E_{\param \sim \dist_t} [L'[\dist_t](\param)]^2}{W_{\epsilon}^2}
	\end{align}
\end{lemma}
Lemma~\ref{lem:posdirdecrease} relies on the 2-homogeneity of $\Phi$ and $R$ and is proven via arguing that the gradient flow dynamics will result in a large shift in $\dist_t$ and therefore substantial decrease in loss. 

For the second case, we will show that the $\sigma U^d$ noise term will cause mass to grow exponentially fast in this descent direction until we make progress in decreasing the objective. 

\begin{lemma}
	\label{lem:negdirdecrease}
	Fix any $\tau > 0$. Choose time interval length $l$ by
	\begin{align*}
	l \ge \frac{\log(W_{\epsilon}^2/\sigma) + 2d \log \frac{2 c_2}{\tau}}{\tau - \sigma} + 1
	\end{align*}
	If $\exists \bparam \in \sphd$ with $L'[\dist_{t^*}](\bparam) \le -\tau$ for some $t^*$ satisfying $t^* + l \le t_{\epsilon}$, then after $l$ steps, we will have
	\begin{align}
	\label{eq:objdecrease}
	L[\dist_{t^* + l}]&\le  L[\dist_{t^*}] -\frac{(\tau/4 - \sigma l c_1 (W_{\epsilon}^2 + 1))^2}{lc_1^2W_{\epsilon}^2}+\sigma l c_1(W_{\epsilon}^2 + 1)
	\end{align}
	Here $c_1$ is the constant defined in Lemma~\ref{lem:fprimechange} and $c_2$ is defined by $c_2 \triangleq \sqrt{k} M_R M_{\Phi} + M_V$. 
\end{lemma}

Lemma~\ref{lem:negdirdecrease} is proven via the following argument: first, if $L'[\rho_t](\bparam)$ is close to $-\tau$ for all $t \in [t^*, t^* + l]$, then from the 2-homogeneity of $\Phi$ and $R$, the mass of $\dist_t$ in the neighborhood around $\bparam$ will grow exponentially fast, leading to a violation of Lemma~\ref{lem:Wbound}. (Because of the uniform noise injected into the gradient flow dynamics, $\dist_t$ will always have some mass in the neighborhood of $\bparam$ to start with.) Thus, it follows that $L'[\rho_t](\bparam)$ must change by at least $\tau/4$, allowing us to invoke Lemma~\ref{lem:fprimechange} to argue that the objective must drop. 

Lemmas~\ref{lem:posdirdecrease} and~\ref{lem:negdirdecrease} are enough to ensure that the objective will always decrease a sufficient amount after some polynomial-size time interval. This allows us to complete the proof of Theorem~\ref{thm:polytimeopt} below:

\begin{proof} [Proof of Theorem \ref{thm:polytimeopt}]
	Let $L^{\star}$ denote the infimum $\inf_{\dist}L[\dist]$, and let $\dist^{\star}$ be an $\epsilon$-approximate global minimizer of $L$: $L[\dist^{\star}] \le L^{\star} + \epsilon$. (We define $\dist^{\star}$ because a true minimizer of $L$ might not exist.) Let $W^{\star} \triangleq \E_{\param \sim \dist^{\star}}[\|\param\|_2^2]$. We first note that since $b_V {W^{\star}}^2 \le L[\dist^{\star}] \le L[\dist_0]$, ${W^{\star}}^2 \le L[\dist_0]/b_V \le W_{\epsilon}^2$. 
	
	Now we bound the suboptimality of $\dist_{t}$: since $L$ is convex in $\dist$, 
	\begin{align*}
	L[\dist^{\star}] \ge L[\dist_{t}] + \E_{\param \sim \dist^{\star}}[L'[\dist_{t}](\param)] - \E_{\param \sim \dist_t}[L'[\dist_t](\param)]
	\end{align*}
	Rearranging gives 
	\begin{align}
	L[\dist_t] - L[\dist^{\star}] &\le \E_{\param \sim \dist_t}[L'[\dist_t](\param)] - \E_{\param \sim \dist^{\star}}[L'[\dist_t](\param)] \nonumber \\
	&\le \E_{\param \sim \dist_t}[L'[\dist_t](\param)] - {W^{\star}}^2 \min\left\{\min_{\bparam \in \sph} L'[\dist_t](\bparam), 0 \right\}
	\label{eq:convexitybound}
	\end{align}
	
	Now let $l \triangleq \frac{W_\epsilon^2}{\epsilon -2 W^2_\epsilon \sigma}\big( 2\log \frac{W_{\epsilon}^2}{\sigma} + 2d  \log \frac{4W_{\epsilon}^2 c_2}{\epsilon}\big)$, which satisfies Lemma \ref{lem:negdirdecrease} with the value of $\tau$ later specified. Suppose that there is a  $t$ with $0\le t \le t_\epsilon -2l$ and $\forall t' \in [t, t + 2l]$, $L[\dist_{t'}] - L^{\star} \ge 2\epsilon$. Then $L[\dist_{t'}] - L[\dist^{\star}] \ge \epsilon$. We will argue that the objective decreases when we are $\epsilon$ suboptimal:
	\begin{align}
	&	L[\dist_{t}] - L[\dist_{t + 2l}] \ge \\& \min\left \{\frac{(\epsilon/8W_{\epsilon}^2 - l\sigma c_1(W_{\epsilon}^2 + 1))^2}{c_1^2W_{\epsilon}^2l} - 3\sigma l c_1(W_{\epsilon}^2 + 1), l \frac{\epsilon^2}{4W_{\epsilon}^2} - 2\sigma lB_L(W_{\epsilon}^2 + 1)\right \}
	\label{eq:objdecrease2ktime}
	\end{align}
	
	Using \eqref{eq:convexitybound} and $W_{\epsilon} \ge W^{\star}$, we first note that
	\begin{align*}
	\epsilon \le \E_{\param \sim \dist_{t'}}[L'[\dist_{t'}](\param)] - {W_{\epsilon}}^2 \min\left\{\min_{\bparam \in \sph} L'[\dist_{t'}](\bparam), 0 \right\} \ \forall t' \in [t, t + l]
	\end{align*}
	Thus, either $\min_{\bparam \in \sphd} L'[\dist_{t'}](\bparam) \le -\frac{\epsilon}{2{W^{\star}}^2} \le -\frac{\epsilon}{2W_{\epsilon}^2}$, or $\E_{\param \sim \dist_{t'}}[L'[\dist_{t'}](\param)] \ge \frac{\epsilon}{2}$. If $\exists t' \in [t, t + l]$ such that the former holds, then we can apply Lemma~\ref{lem:negdirdecrease} with $\tau \triangleq \frac{\epsilon}{2W_{\epsilon}^2}$ to obtain
	\begin{align*}
	L[\dist_{t'}] - L[\dist_{t' + l}] \ge \frac{(\epsilon/8W_{\epsilon}^2 - l\sigma c_1(W_{\epsilon}^2 + 1))^2}{c_1^2W_{\epsilon}^2l} - \sigma l c_1(W_{\epsilon}^2 + 1)
	\end{align*}
	Furthermore, from Lemma \ref{lem:slowincreasebound}, $L[\dist_{t + 2l}] - L[\dist_{t' + l}] \le \sigma l c_1(W_{\epsilon}^2 + 1)$ and $L[\dist_{t'}] - L[\dist_t] \le \sigma l B_L(W_{\epsilon}^2 + 1)$, and so combining gives 
	\begin{align}
	L[\dist_t] - L[\dist_{t + 2l}] \ge \frac{(\epsilon/8W_{\epsilon}^2 - l\sigma c_1(W_{\epsilon}^2 + 1))^2}{c_1^2W_{\epsilon}^2l} - 3\sigma l c_1(W_{\epsilon}^2 + 1)
	\label{eq:case1-progress}
	\end{align}
	In the second case $\E_{\param \sim \dist_{t'}}[L'[\dist_{t'}](\param)] \ge \frac{\epsilon}{2},\ \forall t' \in [t, t + l]$. Therefore, we can integrate \eqref{eq:posdirdecrease} from $t$ to $t + l$ in order to get 
	\begin{align*}
	L[\dist_t] - L[\dist_{t + l}] \ge l \frac{\epsilon^2}{4W_{\epsilon}^2} - \sigma l B_L(W_{\epsilon}^2 + 1)
	\end{align*}
	Therefore, applying Lemma \ref{lem:slowincreasebound} again gives 
	\begin{align}
	L[\dist_t] - L[\dist_{t + 2l}] \ge l \frac{\epsilon^2}{4W_{\epsilon}^2} - 2\sigma lB_L(W_{\epsilon}^2 + 1)
	\label{eq:case2-progress}
	\end{align}
	Thus \eqref{eq:objdecrease2ktime} follows. 
	
	Now recall that we choose
	\begin{align*}
	\sigma \triangleq \exp(-d \log(1/\epsilon) \poly(k, M_V, M_R, M_{\Phi}, b_V, B_V, C_R, B_{\Phi}, L[\dist_0] - L[\dist^{\star}]))
	\end{align*}
	
	For the simplicity, in the remaining computation, we will use $O(\cdot)$ notation to hide polynomials in the problem parameters besides $d, \epsilon$. We simply write $\sigma = \exp(-c_3 d \log(1/\epsilon))$. Recall our choice $t_{\epsilon} \triangleq O(\frac{d^2}{\epsilon^4} \log^2(1/\epsilon))$. It suffices to show that our objective would have sufficiently decreased in $t_{\epsilon}$ steps. We first note that with $c_3$ sufficiently large, $W_{\epsilon}^2 = O(L[\dist_0]/b_v) = O(1)$. Simplifying our expression for $l$, we get that $l = O(\frac{d}{\epsilon} \log \frac{1}{\epsilon})$, so long as $\sigma W_{\epsilon}^2 = o(\epsilon)$, which holds for sufficiently large $c_3$. Now let 
	\begin{align*}
	\delta_1 &\triangleq \frac{(\epsilon/8W_{\epsilon}^2 - l\sigma c_1(W_{\epsilon}^2 + 1))^2}{c_1^2W_{\epsilon}^2l} - 3\sigma l c_1(W_{\epsilon}^2 + 1) \\
	\delta_2 &\triangleq l \frac{\epsilon^2}{4W_{\epsilon}^2} - 2\sigma lB_L(W_{\epsilon}^2 + 1)
	\end{align*}
	Again, for sufficiently large $c_3$, the terms with $\sigma$ become negligible, and $\delta_1 = O(\frac{\epsilon^2}{l}) = O(\frac{\epsilon^3}{d \log(1/\epsilon)})$. Likewise, $\delta_2 = O(d\epsilon \log(1/\epsilon))$.
	
	Thus, if by time $t$ we have not encountered $2\epsilon$-optimal $\dist_t$, then we will decrease the objective by $O(\frac{\epsilon^3}{d \log(1/\epsilon)})$ in $O(\frac{d}{\epsilon} \log \frac{1}{\epsilon})$ time. Therefore, a total of $O(\frac{d^2}{\epsilon^4} \log^2(1/\epsilon))$ time is sufficient to obtain $\epsilon$ accuracy.
\end{proof}

In the following section, we will complete the proofs of Lemmas~\ref{lem:Wbound},~\ref{lem:fprimechange},~\ref{lem:posdirdecrease}, and~\ref{lem:negdirdecrease}. 

\subsection{Missing Proofs for Theorem~\ref{thm:polytimeopt}}
\label{subsec:optproof}
In this section, we complete the proofs of Lemmas~\ref{lem:Wbound},~\ref{lem:fprimechange},~\ref{lem:posdirdecrease}, and~\ref{lem:negdirdecrease}. We first collect some general lemmas which will be useful in these proofs. The following general lemma computes integrals over vector field divergences. 
\begin{lemma}
	\label{lem:divinnerprod}
	For any $h_1 : \R^{d + 1} \rightarrow \R$, $h_2 : \R^{d + 1} \rightarrow \R^{d + 1}$ and distribution $\dist$ with $\dist(\param) \rightarrow 0$ as $\|\param\| \rightarrow \infty$, 
	\begin{align*}
	\int h_1(\param) \nabla \cdot(h_2(\param) \dist(\param))d\param = -E_{\param \sim \dist} [\langle \nabla h_1(\param),  h_2(\param) \rangle]
	\end{align*}
\end{lemma}
\begin{proof}
	The proof follows from integration by parts. 
\end{proof}

We note that $\dist_t$ will satisfy the boundedness condition of Lemma \ref{lem:divinnerprod} during the course of our algorithm - $\dist_0$ starts with this property, and Lemma \ref{lem:Wbound} proves that $\dist_t$ will continue to have this property. We therefore freely apply Lemma \ref{lem:divinnerprod} in the remaining proofs. Now we bound the absolute value of $L'[\dist_t]$ over the sphere by $B_L$. 
\begin{lemma}
	\label{lem:lprimebound}
	For any $\bparam \in \sph, t \ge 0$, $|L'[\dist_t](\bparam)| \le B_L$. 
\end{lemma}
\begin{proof}
	We compute 
	\begin{align*}
	|L'[\dist_t](\bparam)| &= \left|\left\langle \nabla R\left(\int \Phi d\dist\right), \Phi(\bparam)\right\rangle + V(\bparam)\right| \\&\le \left\|\nabla R\left(\int \Phi d\dist\right)\right\|_2 \|\Phi(\bparam)\|_2 + V(\bparam) \le M_R B_{\Phi} + B_V
	\end{align*}
\end{proof}

The next lemma analyzes the decrease in $L[\dist_t]$ due to the gradient flow dynamics. 

\begin{lemma}
	\label{lem:Ltimederiv}
	Under the perturbed Wasserstein gradient flow
	\begin{align*}
	\frac{d}{dt} L[\dist_t] = -\sigma \E_{\param \sim \dist_t}[L'[\dist_t](\param)] + \sigma \E_{\bparam \sim U^d} [L'[\dist_t](\bparam)] - \E_{\param \sim \dist_t}[\|v[\dist_t](\param)\|_2^2]
	\end{align*}
\end{lemma}

\begin{proof}
	Applying the chain rule, we can compute
	\begin{align*}
	\frac{d}{dt} L[\dist_t] &= \left \langle \nabla R\left(\int \Phi d\dist_t\right), \frac{d}{dt} \int \Phi d\dist_t \right \rangle + \frac{d}{dt} \int V d\dist_t \\ 
	&= \frac{d}{dt} \E_{\param \sim \dist_t}[L'[\dist_t](\param)] \\
	&= \int L'[\dist_t](\param) \dist'_t(\param) d\param \\
	&= -\sigma \int L'[\dist_t] d \dist_t + \sigma \int L'[\dist_t] dU^d - \int L'[\dist_t](\param)\nabla \cdot(v[\dist_t](\param)\dist_t(\param)) d\param \\
	&= -\sigma \E_{\param \sim \dist_t}[L'[\dist_t](\param)] + \sigma \E_{\bparam \sim U^d} [L'[\dist_t](\bparam)] - \E_{\param \sim \dist_t}[\|v[\dist_t](\param)\|_2^2]  ,
	\end{align*}
	where we use Lemma \ref{lem:divinnerprod} with $h_1 = L'[\dist_t]$ and $h_2 = v[\dist_t]$.
\end{proof}

By combining the above Lemma with Lemma~\ref{lem:lprimebound}, it follows that at the decrease in objective value is approximately the average velocity of all parameters under $\dist_t$ plus some additional noise on the scale of $\sigma$. At the end, we choose $\sigma$ small enough so that the noise terms essentially do not matter. 
\begin{corollary}
	\label{cor:Ltimederivub}
	We can bound $\frac{d}{dt} L[\dist_t]$ by 
	\begin{align}
	\label{eq:Ltimederivub}
	\frac{d}{dt} L[\dist_t] \le \sigma B_L(W_t^2 + 1) - \E_{\param \sim \dist_t}[\|v[\dist_t](\param)\|_2^2]
	\end{align}
\end{corollary}
\begin{proof}
	By homogeneity, and Lemma \ref{lem:lprimebound}, $\E_{\param \sim \dist_t}[L'[\dist_t](\param)] = \E_{\param \sim \dist_t}[L'[\dist_t](\bparam) \|\param\|_2^2] \le B_L W_t^2$. We also get $\E_{\bparam \sim U^d} [L'[\dist_t](\bparam)] \le B_L$ since $U^d$ is only supported on $\sphd$. Combining these with Lemma \ref{lem:Ltimederiv} gives the desired statement. 
\end{proof} 

Corollary~\ref{cor:Ltimederivub} implies that if we run the dynamics for a short time, the second moment of $\dist_t$ will grow slowly, again at a rate that is roughly the scale of the noise $\sigma$. This allows us to complete the proof of Lemma~\ref{lem:Wbound}. 
\begin{proof}[Proof of Lemma~\ref{lem:Wbound}]
	Let $t^* \triangleq \argmax_{t' \in [0, t]} W_{t'}^2$. Integrating both sides of \eqref{eq:Ltimederivub}, and rearranging, we get 
	\begin{align*}
	0 \le \int_{0}^{t^*}\E_{\param \sim \dist_s}[\|v[\dist_s](\param)\|_2^2] ds &\le  L[\dist_0] - L[\dist_t] + \sigma B_L \int_{0}^{t^*} (W_s^2 + 1)ds \\
	&\le L[\dist_0] - L[\dist_{t^*}] + t^* \sigma B_L(W_{t^*}^2 + 1)
	\end{align*}
	Now since $R$ is nonnegative, we apply $L[\dist_{t^*}] \ge E_{\param \sim \dist_{t^*}} [V(\param)] \ge E_{\param \sim \dist_{t^*}} [V(\bparam)\|\param\|_2^2] \ge b_V W_{t^*}^2$. We now plug this in and rearrange to get $W_{t'}^2 \le W_{t^*}^2 \le \frac{L[\dist_0] + t^* \sigma B_L}{b_V - t^* \sigma B_L} \le \frac{L[\dist_0] + t \sigma B_L}{b_V - t \sigma B_L} \ \forall 0 \le t' \le t$. 
	
	From the proof above, it immediately follows that $\forall 0 \le t \le t_{\epsilon}$, $W_t^2 \le W_{\epsilon}^2$. 
\end{proof}

The next statement allows us to argue that our dynamics will never increase the objective by too much. 
\begin{lemma}
	\label{lem:slowincreasebound}
	For any $t_1, t_2$ with $0 \le t_1 \le t_2 \le t_{\epsilon}$, $L[\dist_{t_2}] - L[\dist_{t_1}] \le \sigma(t_2 - t_1) B_L(W_{\epsilon}^2 + 1)$. 
\end{lemma}
\begin{proof}
	From Corollary \ref{cor:Ltimederivub}, $\forall t \in [t_1, t_2]$ we have
	\begin{align*}
	\frac{d}{dt}L[\dist_{t}] \le \sigma B_L(W_{\epsilon}^2 + 1)
	\end{align*}
	Integrating from $t_1$ to $t_2$ gives the desired result. 
\end{proof}

The following lemma bounds the change in expectation of a 2-homogeneous function over $\dist_t$. At a high level, we lower bound the decrease in our loss as a function of the change in this expectation. By applying this lemma, we will be able to prove Lemma~\ref{lem:fprimechange}. 
\begin{lemma}
	\label{lem:2hom}
	Let $h : \R^{d + 1} \rightarrow \mathbb{R}$ that is 2-homogeneous, with $\|\nabla h(\bparam)\| \le M \ \forall \bparam \in \sphd$ and $|h(\bparam)| \le B \ \forall \bparam \in \sphd$. Then $\forall 0 \le t \le t_{\epsilon}$, we have
	\begin{equation}
	\left|\frac{d}{dt} \int h d\dist_t\right| \le \sigma B(W_{\epsilon}^2 + 1) + MW_{\epsilon} \left(-\frac{d}{dt} L[\dist_t] + \sigma B_L(W_{\epsilon}^2 + 1)\right)^{1/2} \label{eq:timederiv}
	\end{equation}
\end{lemma}

\begin{proof}
	Let $Q(t) \triangleq \int h d\dist_{t}$. We can compute: 
	\begin{align}
	Q'(t) &= \int h(\param) \frac{d\dist_t }{dt} (\param) d\param \nonumber \\
	&= \int h(\param) (-\sigma \dist_t(\param) - \nabla \cdot(v[\dist_t](\param) \dist_t(\param))) d\param + \sigma \int h dU^{d} \nonumber \\
	&= -\sigma \int h(\bparam) \|\param\|_2^2 \dist_t(\param) d\param + \sigma \int h dU^{d} - \int h(\param) \nabla \cdot(v[\dist_t](\param) \dist_t(\param)) d\param \label{eq:sumofints}
	\end{align}
		Note that the first two terms are bounded by $\sigma B(W_{\epsilon}^2 + 1)$ by the assumptions for the lemma. For the third term, we have from Lemma \ref{lem:divinnerprod}: 
	\begin{align*}
	\Big|\int h(\param) \nabla \cdot(v[\dist_t](\param)& \dist_t(\param)) d\param\Big|= |E_{\param \sim \dist_t} [\langle \nabla h(\param), v[\dist_t](\param)\rangle]| \\ 
	&\le \sqrt{E_{\param \sim \dist_t} [\|\nabla h(\param)\|_2^2] E_{\param \sim \dist_t}[\|v[\dist_t](\param)\|_2^2]} \tag{by Cauchy-Schwarz}\\
	&\le \sqrt{E_{\param \sim \dist_t} [\|\nabla h(\bparam)\|_2^2 \|\param\|_2^2]E_{\param \sim \dist_t}[\|v[\dist_t](\param)\|_2^2]}   \tag{by homogeneity of $\nabla h$}  \\
	&\le M W_{\epsilon}\sqrt{E_{\param \sim \dist_t}[\|v[\dist_t](\param)\|_2^2]}  \tag{since $h$ is Lipschitz on the sphere} \\
	&\le M W_{\epsilon} \left(-\frac{d}{dt} L[\dist_t] + \sigma B_L(W_{\epsilon}^2 + 1)\right)^{1/2}   \tag{by Corollary \ref{cor:Ltimederivub}}
	\end{align*}
	Plugging this into \eqref{eq:sumofints}, we get that
	\begin{align*}
	|Q'(t)| \le \sigma B(W_{\epsilon}^2 + 1) + MW_{\epsilon} \left(-\frac{d}{dt} L[\dist_t] + \sigma B_L(W_{\epsilon}^2 + 1)\right)^{1/2} 
	\end{align*}
\end{proof}


Now we complete the proof of Lemma~\ref{lem:fprimechange}. 
\begin{proof}[Proof of Lemma~\ref{lem:fprimechange}]
	Recall that $L'[\dist_t](\bparam) = \langle \nabla R(\int \Phi d\dist_t), \Phi(\bparam) \rangle + V(\bparam)$. Differentiating with respect to $t$, 
	\begin{align}
	\frac{d}{dt} L'[\dist_t](\bparam) &= \left\langle \frac{d}{dt} \nabla R\left(\int \Phi d\dist_t\right), \Phi(\bparam) \right \rangle \nonumber \\
	&= \Phi(\bparam)^{\top} \nabla^2 R(Q(t)) Q'(t) \nonumber \\
	&\le C_R B_{\Phi} \|Q'(t)\|_2 \nonumber \\ 
	&\le C_R B_{\Phi} \|Q'(t)\|_1 \label{eq:l1normbound}
	\end{align}
	Integrating and applying the same reasoning to $-L'[\dist_t]$ gives us \eqref{eq:Lprimechangeub}. Now we apply Lemma \ref{lem:2hom} to get 
	\begin{align*}
	\|Q'(t)\|_1 &= \sum_{i = 1}^k \left|\frac{d}{dt} \int \Phi_i d\dist_t\right|\\
	&\le \sum_{i = 1}^k \left[\sigma B_{\Phi}(W_{\epsilon}^2 + 1) + M_{\Phi}W_{\epsilon}\left(-\frac{d}{dt} L[\dist_t] + \sigma B_L(W_{\epsilon}^2 + 1)\right)^{1/2}\right]\\
	&\le k\sigma B_{\Phi}(W_{\epsilon}^2 + 1) + kM_{\Phi
	}W_{\epsilon}\left(-\frac{d}{dt} L[\dist_t]+ \sigma B_L(W_{\epsilon}^2 + 1)\right)^{1/2}
	\end{align*}
	We plug this into \eqref{eq:l1normbound} and then integrate both sides to obtain 
	\begin{align*}
	&C_R B_{\Phi} \int_{t}^{t + l} \|Q'(t)\|_1 \\&\le k\sigma l C_R B_{\Phi}^2(W_{\epsilon}^2 + 1) + k C_R B_{\Phi}M_{\Phi} W_{\epsilon} \int_t^{t + l} \left(-\frac{d}{dt} L[\dist_t]+ \sigma B_L(W_{\epsilon}^2 + 1)\right)^{1/2}\\
	&\le k\sigma l C_R B_{\Phi}^2(W_{\epsilon}^2 + 1) + k C_{R} B_{\Phi} M_{\Phi} W_{\epsilon} \sqrt{l} (L[\dist_t] - L[\dist_{t + l}] + \sigma l B_L(W_{\epsilon}^2 + 1))^{1/2}
	\end{align*}
	Using $c_1 \triangleq \max \{kC_R B_{\Phi}^2, k C_RB_{\Phi}M_{\Phi}, B_L\}$ gives the statement in the lemma.
\end{proof}

Now we will fill in the proof of Lemma~\ref{lem:negdirdecrease}. 
We first show that $L'$ is Lipschitz on the unit ball. Recall that in the statement of Lemma~\ref{lem:negdirdecrease}, we define a constant $c_2$ by $c_2 \triangleq \sqrt{k} M_R M_{\Phi} + M_V$.
\begin{lemma}
	\label{lem:lprimelipschitz}
	For all $\bparam, \bparam' \in \sphd$,  
	\begin{align}
	|L'[\dist](\bparam) - L'[\dist](\bparam')| \le c_2 \|\bparam - \bparam'\|_2 \label{eq:lprimelipschitz}
	\end{align}
\end{lemma}
\begin{proof}
	Using the definition of $L'$ and triangle inequality,
	\begin{align*}
	|L'[\dist](\bparam) - L'[\dist](\bparam')| &\le \left\|\nabla R\left(\int \Phi d\dist\right)\right\|_2 \| \Phi(\bparam) - \Phi(\bparam')\|_2 + |V(\bparam) - V(\bparam')| \\
	&\le (\sqrt{k} M_R M_{\Phi} + M_V)\| \bparam - \bparam'\|_2 \tag{by definition of $M_{\Phi}, M_R, M_V$}
	\end{align*}

\end{proof}

Next, we introduce notation to refer to the $-\tau$-sublevel set of $L'[\dist_t]$. This will be useful for our proof of Lemma~\ref{lem:negdirdecrease}. Define $K_t^{-\tau} \triangleq \{\bparam \in \sphd : L'[\rho_t](\bparam) \le -\tau\}$, the $-\tau$-sublevel set of $L'[\dist_t]$, and let $m(S) \triangleq \E_{\param \sim U^{d}}[\1(\param \in S)]$ be the normalized spherical area of the set $S$. The following statement uses the Lipschitz-ness of $L'[\dist_t]$ to lower bound the volume of $K_t^{-\tau + \delta}$ for some $\delta > 0$ if the $-\tau$-sublevel set $K_t^{-\tau}$ is nonempty.

\begin{lemma}
	\label{lem:vollb}
	If $K_t^{-\tau}$ is nonempty, for $0 \le \delta \le \tau$, $\log m(K_t^{-\tau + \delta}) \ge -2d\log \frac{c_2}{\delta}$. 
\end{lemma}
\begin{proof}
	Let $\bparam \in K_t^{-\tau}$. From Lemma \ref{lem:lprimelipschitz}, $L'[\dist](\bparam') \le -\tau + \delta$ for all $\bparam'$ with $\|\bparam' - \bparam\|_2 \le \frac{\delta}{c_2}$. Thus, we have 
	\begin{align*}
	m(K_t^{-\tau + \delta}) \ge \E_{\bparam' \sim U^d}\left[\ind{\|\bparam' - \bparam\|_2 \le \frac{\delta}{c_2}}\right]
	\end{align*}
	Now the statement follows by Lemma 2.3 of \citep{ball1997elementary}.
\end{proof}       

Finally, the proof of Lemma~\ref{lem:negdirdecrease} will require a general lemma about the magnitude of the gradient of a 2-homogeneous function in the radial direction. 

\begin{lemma}
	Let $h : \R^{d + 1} \rightarrow \mathbb{R}$ be a 2-homogeneous function. Then for any $\param \in \R^{d + 1}$, $\bparam^{\top} \nabla h(\param) = 2 \|\param\|_2 h(\bparam)$.
	\label{lem:2homgradient}
\end{lemma}
\begin{proof}
	We have $h(\param + \alpha \bparam) = (\|\param\|_2 + \alpha)^2h(\bparam)$. Differentiating both sides with respect to $\alpha$ and evaluating the derivative at 0, we get $\bparam^{\top} \nabla h(\param) = 2\|\param\|_2 h(\bparam)$, as desired. 
\end{proof}


Now we are ready to complete the proof of Lemma~\ref{lem:negdirdecrease}. Recall that in the setting of Lemma~\ref{lem:negdirdecrease}, $l$ is the length of the time interval over which the descent direction causes a decrease in the objective. We will first show that a descent direction in $L'[\dist_t]$ will remain for the next $l$ time steps. In the notation of Lemma \ref{lem:fprimechange}, define $z(s) \triangleq C_R B_{\Phi}\int_{t^*}^{t^* + s} \|Q'(t)\|_1dt$. Note that from Lemma \ref{lem:fprimechange}, for all $\bparam \in \sphd$  we have $|L'[\dist_{t^* + s}](\bparam) - L'[\dist_{t^*}](\bparam)| \le z(s)$. Thus, the following holds: 

\begin{claim}
	\label{claim:Ktauzsnonempty}
	For all $s \le l$, $K_{t^* + s}^{-\tau + z(s)}$ is nonempty. 
\end{claim}
\begin{proof}
	By assumption, $\exists \bparam$ with $\bparam \in K_{t^*}^{-\tau}$. Then $L'[\dist_{t^* + s}](\bparam) \le L'[\dist_{t^*}](\bparam) + z(s) \le -\tau + z(s)$, so $K_{t^* + s}^{-\tau + z(s)}$ is nonempty. 
\end{proof}

Let $T_{s} \triangleq K_{t^* + s}^{-\tau/2 + z(s)}$ for $0 \le s \le l$. We now argue that this set $T_{s}$ does not shrink as $t$ increases. 

\begin{claim}
	For all $s' > s$, $T_{s'} \supseteq T_{s}$. 
\end{claim}
\begin{proof}
	From \eqref{eq:l1normbound} and the definition of $z(s)$, $|L'[\dist_{t + s'}](\bparam) - L'[\dist_{t + s}](\bparam)| \le z(s') - z(s)$. It follows that for $\bparam \in T_s$
	\begin{align*}
	L'[\dist_{t + s'}](\bparam) &\le L'[\dist_{t + s}](\bparam) + z(s') - z(s)\\
	&\le -\tau/2 + z(s)- z(s) + z(s') & \text{(by definition of }T_s) \\
	&\le -\tau/2 + z(s')
	\end{align*}
	which means that $\bparam \in T_{s'}$. 
\end{proof}

Now we show that the weight of the particles in $T_{s}$ grows very fast if $z(k)$ is small.
\begin{claim}
	\label{claim:fastgrowth}
	Suppose that $z(l) \le \tau/4$. Let $\tilde{T}_{s} = \{\param \in \R^{d + 1} : \bparam \in T_{s}\}$. Define $N(s) \triangleq \int_{\tilde{T}_{s}} \|\param\|^2 d \dist_{t^* + s}$ and $\beta \triangleq \exp(-2d\log \frac{2c_2}{\tau})$. Then $N'(s) \ge (\tau - \sigma)N(s) + \sigma \beta$. 
\end{claim}

\begin{proof}
	From the assumption $z(l) \le \frac{\tau}{4}$, it holds that $T_{s} \subseteq K_{t^* + s}^{-\tau/4} \ \forall s \le k$. Since $T_{s}$ is defined as a sublevel set, $v[\dist_{t^* + s}](\bparam)$ points inwards on the boundary of $T_s$ for all $\bparam \in T_s$, and by 1-homogeneity of the gradient, the same must hold for all $u \in \tilde{T}_{s}$. 
	
	Now consider any particle $\param \in \tilde{T}_{s}$. We have that $\param$ flows to $\param + v[\dist_{t^* + s}](\param)ds$ at time $t^* + s + ds$. Furthermore, since the gradient points inwards from the boundary, it also follows that $u + v[\dist_{t^* + s}](\param) ds \in \tilde{T}_{s}$. Now we compute  
	\begin{align}
	\int_{\tilde{T}_{s}} \|\param\|_2^2 d\dist_{t^* + s + ds} &= (1 - \sigma ds)\int_{\tilde{T}_{s}} \|\param + v[\dist_{t^* + s}](\param)ds\|_2^2 d\dist_{t^* + s} + \sigma ds \int_{\tilde{T}_s} 1dU^{d}\nonumber \\
	&\ge (1 - \sigma ds)\int_{\tilde{T}_{s}} (\|\param\|_2^2 + 2\param^{\top} v[\dist_{t^* + s}](\param) ds) d\dist_{t^* + s} + \sigma m(K^{-\tau/2 + z(s)}_{t^* + s})ds\label{eq:particleflow}
	\end{align}
	Now we apply Lemma \ref{lem:2homgradient}, using the 2-homogeneity of $F'$ and the fact that $L'[\dist_{t^* + s}](\bparam) \le -\tau/4 \ \forall \param \in \tilde{T}_s$
	\begin{align}
	\|\param\|_2^2 +2\param^{\top} v[\dist_{t^* + s}](\param) ds &= \|\param\|_2^2 - 4\|\param\|_2^2 L'[\dist_{t^* + s}](\bparam)ds \nonumber \\
	&\ge \|\param\|_2^2(1 + \tau ds)   \label{eq:normgrowth}
	\end{align}
	Furthermore, since $K_{t^* + s}^{-\tau + z(s)}$ is nonempty by Claim \ref{claim:Ktauzsnonempty}, we can apply Lemma \ref{lem:vollb} and obtain 
	\begin{align}
	\label{eq:mKtaulowerbound}
	m(K^{-\tau/2 + z(s)}_{t^* + s}) \ge \beta 
	\end{align}
	Plugging \eqref{eq:normgrowth} and \eqref{eq:mKtaulowerbound} back into \eqref{eq:particleflow}, we get 
	\begin{align*}
	\int_{\tilde{T}_{s}} \|u\|_2^2 d\dist_{t^* + s + ds} &\ge (1 - \sigma ds)(1 + 2\tau ds) N(s) + \sigma \beta ds
	\end{align*}
	Since we also have that $\tilde{T}_{s + ds} \supseteq \tilde{T}_{s}$, it follows that 
	\begin{align*}
	N(s + ds) = \int_{\tilde{T}_{s + ds}} \|u\|_2^2 d\dist_{t^* + s + ds}
	\ge (1 - \sigma ds)(1 + \tau ds) N(s) + \sigma \beta ds
	\end{align*}
	and so $N'(s) \ge (\tau - \sigma) N(s) + \sigma \beta$.
\end{proof}

Using Claim~\ref{claim:fastgrowth} allows us to complete the proof of Lemma \ref{lem:negdirdecrease}.
\begin{proof}[Proof of Lemma \ref{lem:negdirdecrease}]
	If $z(l) = C_R B_\Phi \int_{t} ^{t+l} \| Q'(t) \|_1  \ge \frac{\tau}{4}$, then by rearranging the conclusion of  Lemma \ref{lem:fprimechange} we immediately get \eqref{eq:objdecrease}.	
	
	Suppose for the sake of contradiction that $z(l) \le \tau/4$. From Claim \ref{claim:fastgrowth}, it follows that $N(1) \ge \sigma \beta$, and $N(l) \ge \exp((\tau - \sigma)(l - 1)) N(1)$. Thus, in $\frac{\log(W_{\epsilon}^2/\sigma) + 2d \log \frac{2 c_2}{\tau}}{\tau - \sigma} + 1$ time, $W_{t^* + l} \ge N(l) \ge W_{\epsilon}^2$, a contradiction. Therefore, it must be true that $z(l) \ge \tau/4$.

\end{proof}

Finally, we fill in the proof of Lemma~\ref{lem:posdirdecrease}. 
\begin{proof}
	We can first compute
	\begin{align*}
	\E_{\param \sim \dist_t} [L'[\dist_t](\param)] &= \E_{\param \sim \dist_t} [L'[\dist_t](\bparam)\|\param\|_2^2] \\ 
	&= \frac{1}{2}\E_{\param \sim \dist_t} [\|\param\|_2 \bparam^{\top} v[\dist_t](\param)] &\tag{via Lemma \ref{lem:2homgradient}} \\
	&\le \frac{1}{2} \sqrt{\E_{\param \sim \dist_t} [\|\param\|_2^2]\E_{\param \sim \dist_t} [\|v[\dist_t](\param)\|_2^2]} &\tag{by Cauchy-Schwarz} \\
	&\le \frac{1}{2}W_{\epsilon} \sqrt{\E_{\param \sim \dist_t} [\|v[\dist_t](\param)\|_2^2]}
	\end{align*}
	Rearranging gives $\E_{\param \sim \dist_t} [\|v[\dist_t](\param)\|_2^2] \ge \frac{\E_{\param \sim \dist_t} [L'[\dist_t](\param)]^2}{W_{\epsilon}^2}$, and plugging this into \eqref{eq:Ltimederivub} gives the desired result. 
\end{proof}

\subsection{Discrete-Time Optimization}
\label{subsec:discrete-time}

To circumvent the technical issue of existence of a solution to the continuous-time dynamics, we also note that polynomial time convergence holds for discrete-time updates. 
\begin{theorem}
Along with Assumptions~\ref{assu:R},~\ref{assu:Phi},~\ref{assu:V} additionally assume that $\nabla \Phi_i$ and $\nabla V$ are $C_{\Phi}$ and $C_{V}$-Lipschitz, respectively. Let $\dist_t$ evolve according to the following discrete-time update:
\begin{align*}
\dist_{t + 1} \triangleq \dist_t + \eta(-\sigma \dist_t + \sigma U^d - \nabla \cdot(v[\dist_t] \dist_t) )
\end{align*}
There exists a choice of 
\begin{align*}
	\sigma &\triangleq \exp(-d \log(1/\epsilon)\poly(k, M_V, M_R, b_V, B_V, C_R, B_{\Phi}, C_{\Phi}, C_V, L[\dist_0] - L[\dist^{\star}])) \\
	\eta &\triangleq \poly(k, M_V, M_R, b_V, B_V, C_R, B_{\Phi}, C_{\Phi}, C_V, L[\dist_0] - L[\dist^{\star}]) \\
	t_{\epsilon} &\triangleq \frac{d^2}{\epsilon^4} \poly(k, M_V, M_R, b_V, B_V, C_R, B_{\Phi}, C_{\Phi}, C_V, L[\dist_0] - L[\dist^{\star}])
\end{align*}
such that $\min_{0 \le t \le t_{\epsilon}} L[\dist_t] - L^{\star} \le \epsilon$. 
\end{theorem}
The proof follows from a standard conversion of the continuous-time proof of Theorem~\ref{thm:polytimeopt} to discrete time, and we omit it here for simplicity.

%% file: additional_experiments.tex
\section{Additional Simulations}
\label{sec:additionalexp}

In this section we provide more details on the simulations described in Section~\ref{sec:experiments}. The experiments were small enough to run on a standard computer, though we used a single NVIDIA TitanXp GPU. We decided the value of regularization $\lambda$ based on the training length - longer training time meant we could use smaller $\lambda$.

\subsection{Test Error and Margin vs. Hidden Layer Size}
\label{subsec:varyhiddendetail}
To justify Theorem~\ref{thm:marginnondecreasing}, we also plot the dependence of the test error and margin on the hidden layer size in Figure~\ref{fig:varyhidden} for synthetic data generated from a ground truth network with $10$ hidden units and also MNIST. The plots indicate that test error is decreasing in hidden layer size while margin is increasing, as Theorem~\ref{thm:marginnondecreasing} predicts. We train the networks for a long time in this experiment: we train for 80000 passes on the synthetic data and 600 epochs for MNIST. 
\begin{figure}
	\centering
			\includegraphics[width=0.24\textwidth]{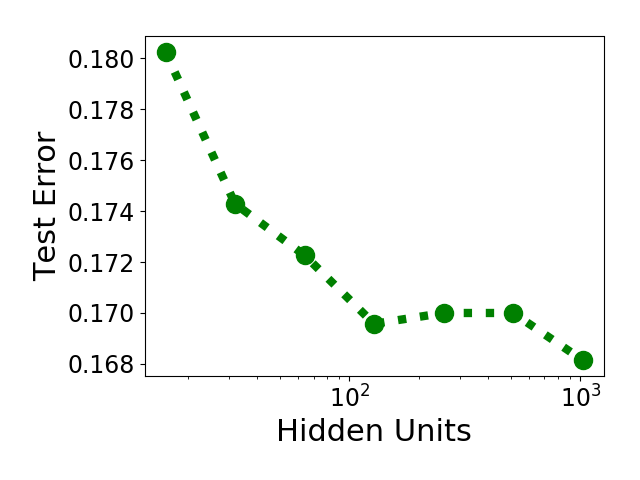}
	\includegraphics[width=0.24\textwidth]{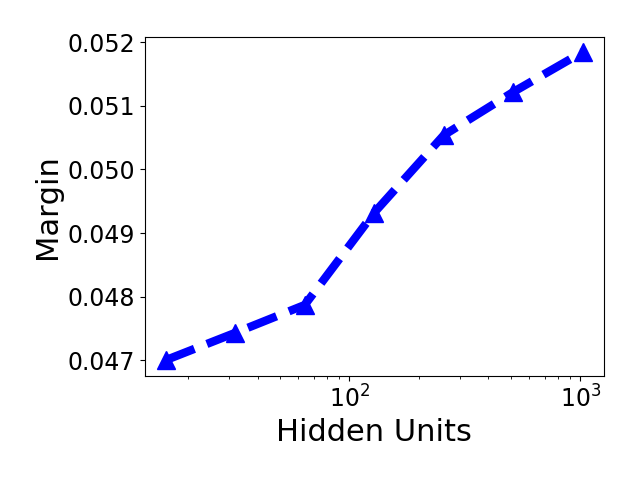}
	\includegraphics[width=0.24\textwidth]{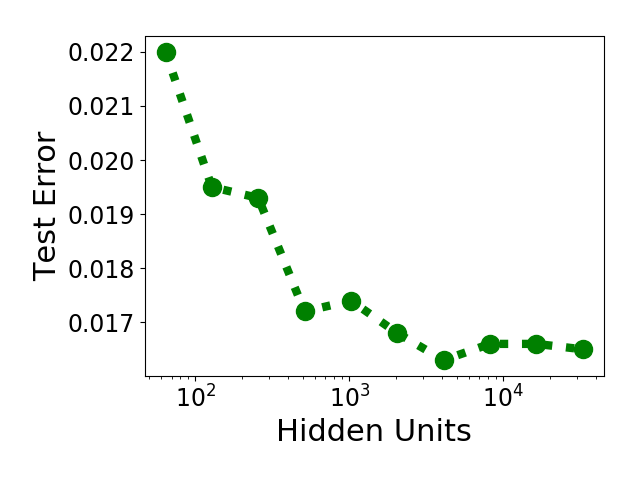}
	\includegraphics[width=0.24\textwidth]{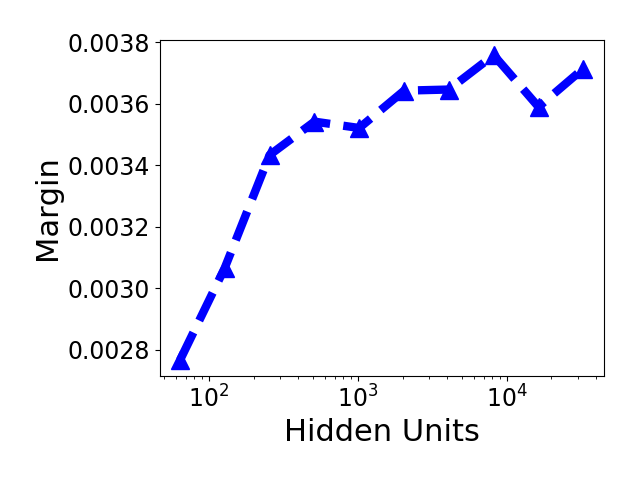}
		\caption{Dependence of margin and test error on hidden layer size. \textbf{Left:} Synthetic. \textbf{Right:} MNIST.}
	\label{fig:varyhidden}
\end{figure}

The left side of Figure~\ref{fig:varyhidden} shows the experimental results for synthetic data generated from a ground truth network with $10$ hidden units, input dimension $d = 20$, and a ground truth unnormalized margin of at least $0.01$. We train for 80000 steps with learning rate $0.1$ and $\lambda = 10^{-5}$, using two-layer networks with $2^i$ hidden units for $i$ ranging from 4 to 10. We perform 20 trials per hidden layer size and plot the average over trials where the training error hit 0. (At a hidden layer size of $2^7$ or greater, all trials fit the training data perfectly.) The right side of Figure~\ref{fig:varyhidden} demonstrates the same experiment, but performed on MNIST with hidden layer sizes of $2^i$ for $i$ ranging from $6$ to $15$. We train for 600 epochs using a learning rate of 0.01 and $\lambda = 10^{-6}$ and use a single trial per plot point. For MNIST, all trials fit the training data perfectly. The MNIST experiments are more noisy because we run one trial per plot point for MNIST, but the same trend of decreasing test error and increasing margin still holds. 

\subsection{Neural Net and Kernel Generalization vs. Training Set Size}
\label{subsec:kernel_vs_nn_exp}

\begin{figure}
		\centering
	\includegraphics[width=0.40\textwidth]{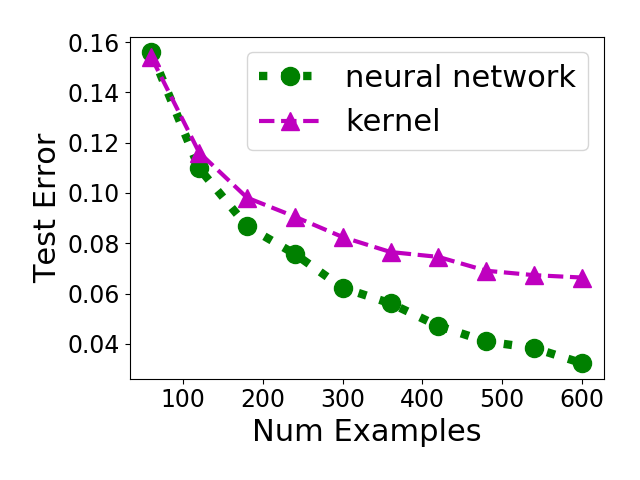}
		\includegraphics[width=0.40\textwidth]{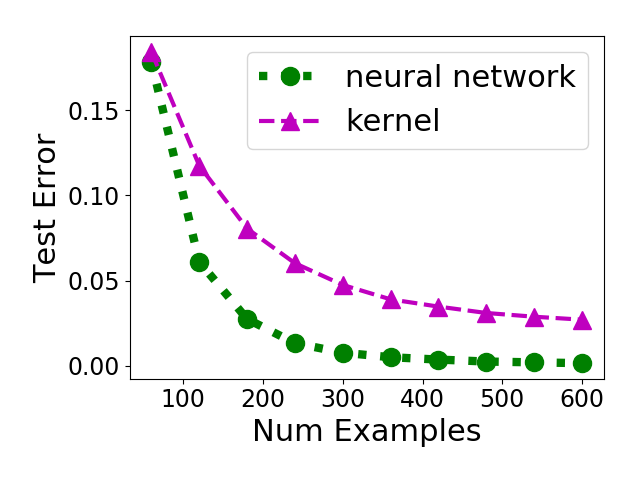}

					\caption{Neural nets vs. kernel method with $r_w=0,r_u=1$ (Theorem~\ref{thm:comparison} setting). \textbf{Left:} Classification. \textbf{Right:} Regression.}
	\label{fig:syntheticgen}
\end{figure}
We compare the generalization of neural nets and kernel methods for classification and regression. In Figure~\ref{fig:syntheticgen} we plot the generalization error of a trained neural net against a $\ell_2$ kernel method with relu features (corresponding to $r_1 = 0, r_2 = 1$ in the setting of Theorem~\ref{thm:comparison}) as we vary $n$. Our ground truth comes from a random neural network with 6 hidden units, and during training we use a network with as many hidden units as examples. For classification, we used rejection sampling to obtain datapoints with unnormalized margin of at least 0.1 on the ground truth network. We use a fixed dimension of $d=20$. For all experiments, we train the network for 20000 steps with $\lambda = 10^{-8}$ and average over 100 trials for each plot point. 

For classification we plot 0-1 error, whereas for regression we plot squared error. The plots show that two-layer nets clearly outperform the kernel method in test error as $n$ grows.

\subsection{Verifying Convergence to the Max-Margin}
\label{subsec:maxmarginverify}
We verify the normalized margin convergence on a two-layer networks with one-dimensional input. A single hidden unit computes the following: $x \mapsto a_j \relu(w_j x + b_j)$. We add $\|\cdot \|_2^2$-regularization to $a, w$, and $b$ and compare the resulting normalized margin to that of an approximate solution of the $\ell_1$ SVM problem with features $\relu(wx_i + b)$ for $w^2 + b^2 = 1$. Writing this feature vector is intractable, so we solve an approximate version by choosing 1000 evenly spaced values of $(w, b)$. Our theory predicts that with decreasing regularization, the margin of the neural network converges to the $\ell_1$ SVM objective. In Figure \ref{fig:inputdim1}, we plot this margin convergence and visualize the final networks and ground truth labels. The network margin approaches the ideal one as $\lambda \rightarrow 0$, and the visualization shows that the network and $\ell_1$ SVM functions are extremely similar.  

\begin{figure}
	\centering
	\includegraphics[width=0.4\textwidth]{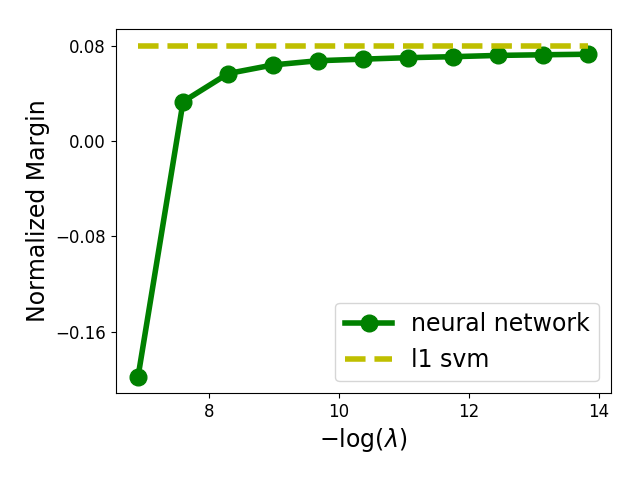}
	\includegraphics[width=0.4\textwidth]{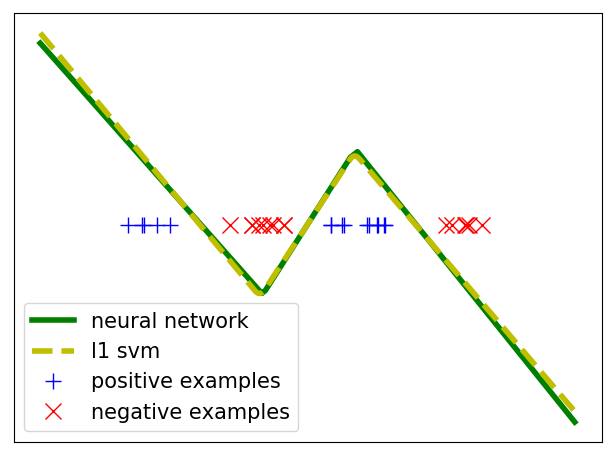}
	\caption{Neural network with input dimension 1. \textbf{Left:} Normalized margin as we decrease $\lambda$. \textbf{Right:} Visualization of the normalized functions computed by the neural network and $\ell_1$ SVM solution for $\lambda \approx 2^{-14}$.}
	\label{fig:inputdim1}
\end{figure}